\documentclass[twocolumn]{autart}    

\newif\ifarxiv

\arxivtrue 

\usepackage{cite}
\usepackage{amsmath,amssymb,amsfonts}

\usepackage{graphicx}
\usepackage{textcomp}

\usepackage[utf8]{inputenc} 
\usepackage[T1]{fontenc}    

\usepackage{url}            
\usepackage{booktabs}       
\usepackage{nicefrac}       

\usepackage{bbm}
\usepackage[acronym]{glossaries}
\usepackage[dvipsnames, table]{xcolor}         
\usepackage{subcaption}
\usepackage{hyperref}

\usepackage{rotating}

\usepackage{hyperref}
\usepackage{comment}

\usepackage{algorithm}
\usepackage[noend]{algpseudocode}
\usepackage{etoolbox}

\newcommand*{\algrule}[1][\algorithmicindent]{%
  \makebox[#1][l]{%
    \hspace*{.2em}
    \vrule height .75\baselineskip depth .25\baselineskip
  }
}

\makeatletter
\newcount\ALG@printindent@tempcnta
\def\ALG@printindent{%
    \ifnum \theALG@nested>0
    \ifx\ALG@text\ALG@x@notext
    \else
    \unskip
    \ALG@printindent@tempcnta=1
    \loop
    \algrule[\csname ALG@ind@\the\ALG@printindent@tempcnta\endcsname]%
    \advance \ALG@printindent@tempcnta 1
    \ifnum \ALG@printindent@tempcnta<\numexpr\theALG@nested+1\relax
    \repeat
    \fi
    \fi
}
\patchcmd{\ALG@doentity}{\noindent\hskip\ALG@tlm}{\ALG@printindent}{}{\errmessage{failed to patch}}
\patchcmd{\ALG@doentity}{\item[]\nointerlineskip}{}{}{} 
\makeatother

\newcommand{\ml}[1]{{\color{Bittersweet}[ML: #1]}}

\newcommand{\unsafe}{\mathrm{u}}
\newcommand{\safe}{\mathrm{s}}

\newcommand{\indicator}[1]{\mathbbm{1}_{#1}}

\newcommand{\expect}{\mathbb{E}}
\newcommand{\pv}{\mathbf{v}}
\newcommand{\px}{\mathbf{x}}

\newcommand{\pX}{\mathbf{x}}

\newcommand{\reals}{\mathbb{R}}
\newcommand{\naturals}{\mathbb{N}}
\newcommand{\N}{\mathcal{N}}

\renewcommand{\P}{\mathcal{P}}
\renewcommand{\O}{\mathcal{O}}
\newcommand{\C}{\mathbbmss{O}}
\renewcommand{\L}{\mathcal{L}}

\newcommand{\up}[1]{\overline{#1}}
\newcommand{\low}[1]{\underline{#1}}

\newtheorem{lemma}{Lemma}
\newtheorem{definition}{Definition}
\newtheorem{problem}{Problem}

\newtheorem{proposition}{Proposition}
\newtheorem{corollary}{Corollary}

\newtheorem{theorem}{Theorem}

\newenvironment{proof}{\textit{Proof}.}{\hfill$\square$}
\newacronym{lp}{LP}{Linear Programming}
\newacronym{cegs}{CEGS}{Counter-Example Guided Synthesis}
\newacronym{gd}{GD}{Gradient Descent}
\newacronym{pwb}{PWB}{Piecewise Barrier}
\newacronym{pwc}{PWC}{Piecewise Constant}
\newacronym{pwa}{PWA}{Piecewise Affine}
\newacronym{sos}{SOS}{Sum-of-Squares}
\newacronym{nbf}{NBF}{Neural Barrier Function}
\newacronym{nndm}{NNDM}{Neural Network Dynamic Model}
\newacronym{imdp}{IMDP}{Interval Markov Decision Process}

\begin{document}

\begin{frontmatter}

\title{
{Piecewise Stochastic Barrier Functions}
} 

\thanks[footnoteinfo]{
Corresponding author: Rayan Mazouz 
}

\author[Boulder,equal]{Rayan Mazouz}\ead{rayan.mazouz@colorado.edu}, 
\author[Delft,equal]{Frederik Baymler Mathiesen},    
\author[Delft]{Luca Laurenti},
\author[Boulder]{Morteza Lahijanian}

\address[Boulder]{Department of Aerospace Engineering Sciences, University of Colorado Boulder, USA}                                         
\address[Delft]{Delft Center for Systems and Control, Delft University of Technology, The Netherlands}

\address[equal]{Equal contributions}


\begin{abstract}                          
This paper presents a novel stochastic barrier function (SBF) framework for safety analysis of stochastic systems based on piecewise (PW) functions. We first outline a general formulation of PW-SBFs. Then, we focus on PW-Constant (PWC) SBFs and show how their simplicity yields computational advantages for general stochastic systems. Specifically, we prove that synthesis of PWC-SBFs reduces to a minimax optimization problem.
Then, we introduce three efficient algorithms to solve this problem, each offering distinct advantages and disadvantages.
The first algorithm is based on dual linear programming (LP), which provides an exact solution to the minimax optimization problem. The second is a more scalable algorithm based on iterative counter-example guided synthesis, which involves solving two smaller LPs. The third algorithm solves the minimax problem using gradient descent, which admits even better scalability. We provide an extensive evaluation of these methods on various case studies, including neural network dynamic models, nonlinear switched systems, and high-dimensional linear systems. Our benchmarks demonstrate that PWC-SBFs outperform state-of-the-art methods, namely sum-of-squares and neural barrier functions, and can scale to eight dimensional systems.
\end{abstract}

\end{frontmatter}

\section{Introduction}
Autonomous systems deployed in \emph{safety-critical} domains demand thorough verification to ensure reliability. Examples of such systems include autonomous vehicles in human-shared streets~\cite{shalev2017formal}, aerospace systems (from aerial transportation~\cite{lee2016planning} to space exploration~\cite{reed2024shielded, mazouz2021dynamics}), and surgical robotics~\cite{guiochet2017safety}. However, this verification process faces major challenges arising from the complexities in dynamics as well as uncertainties due to the physics or randomized algorithms of the system.
One widely employed method to provide necessary safety guarantees is \textit{stochastic barrier functions} (SBFs)~\cite{kushner1967stochastic,  prajna2007framework, santoyo2021barrier}.
SBFs provide a lower bound on safety probabilities, hence enabling \emph{safety certification}.  
Despite the extensive studies on SBFs, existing tools largely remain limited to simplistic settings, namely simple dynamics with low dimensions. 
This work aims to address these limitations by focusing on the development of a theoretical and computational framework for barrier certificate generation tailored to complex systems.

The literature presents two prominent methods for synthesizing SBFs: \gls{sos} optimization~\cite{santoyo2021barrier, mazouz2022safety} and  \gls{nbf}~\cite{mathiesen2022safety, dawson2022safe}. \gls{sos} optimization relies on a predefined \gls{sos} polynomial template for the SBF and aims to determine its coefficients using convex optimization. While powerful, this method suffers when the system dimensionality is large \cite{ahmadi2014dsos} or the safe set is non-convex. On the other hand, \gls{nbf} employ neural networks (NNs) to act as barrier functions, addressing the conservatism of \gls{sos} in non-convex settings. Yet, \gls{nbf} face scalability issues due to the need to verify the NN against the SBF conditions, which is particularly challenging for high-dimensional systems.

\begin{figure*}[t!]
    \centering
    \begin{minipage}{\textwidth}
        \begin{subfigure}[t]{0.25\textwidth}
            \centering
            \includegraphics[width=1\linewidth]{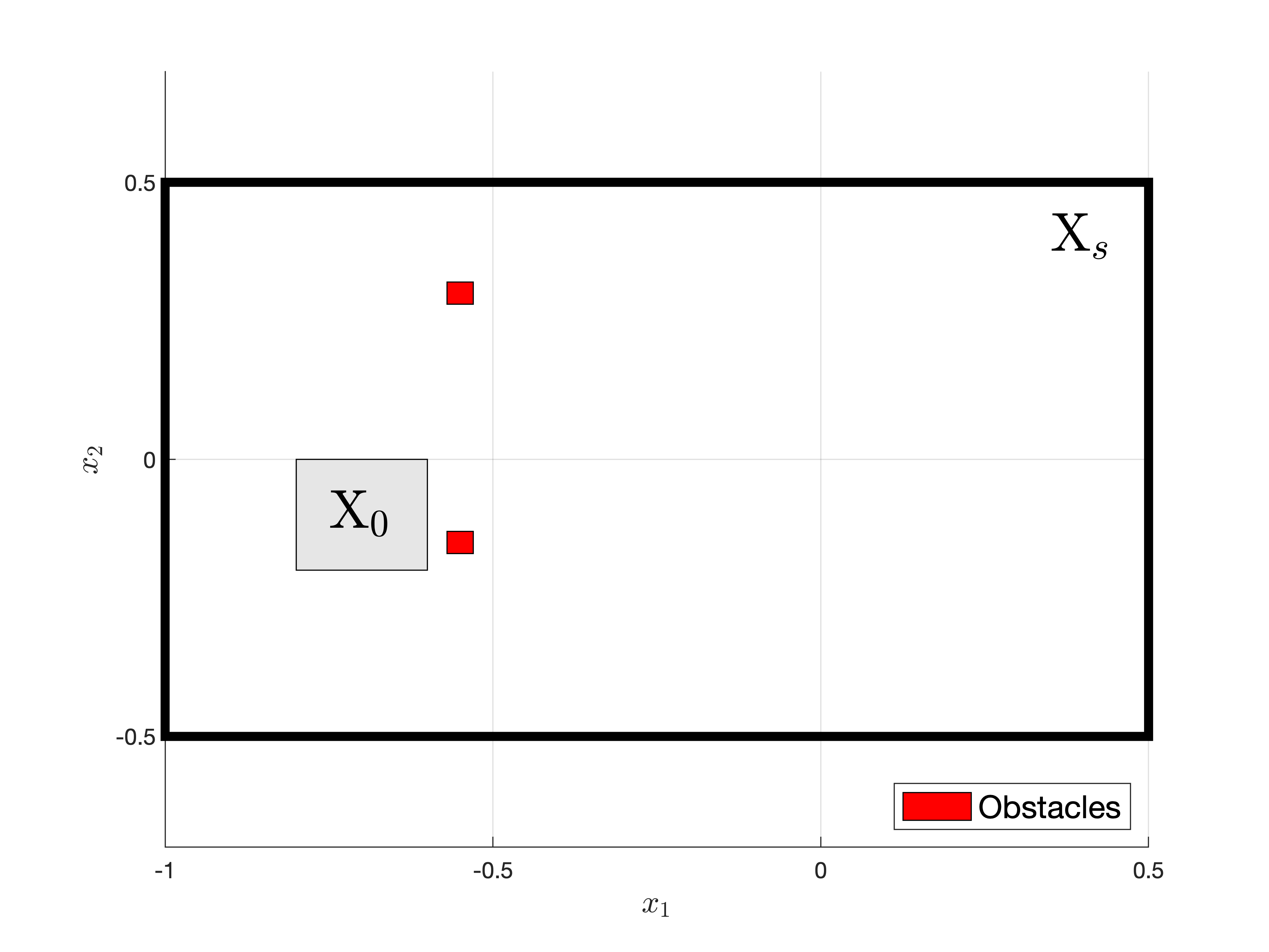}
            \subcaption{Initial and safe sets}
                   \label{fig:env}
        \end{subfigure}%
        \begin{subfigure}[t]{0.25\textwidth}
            \centering
            \includegraphics[width=1\linewidth]{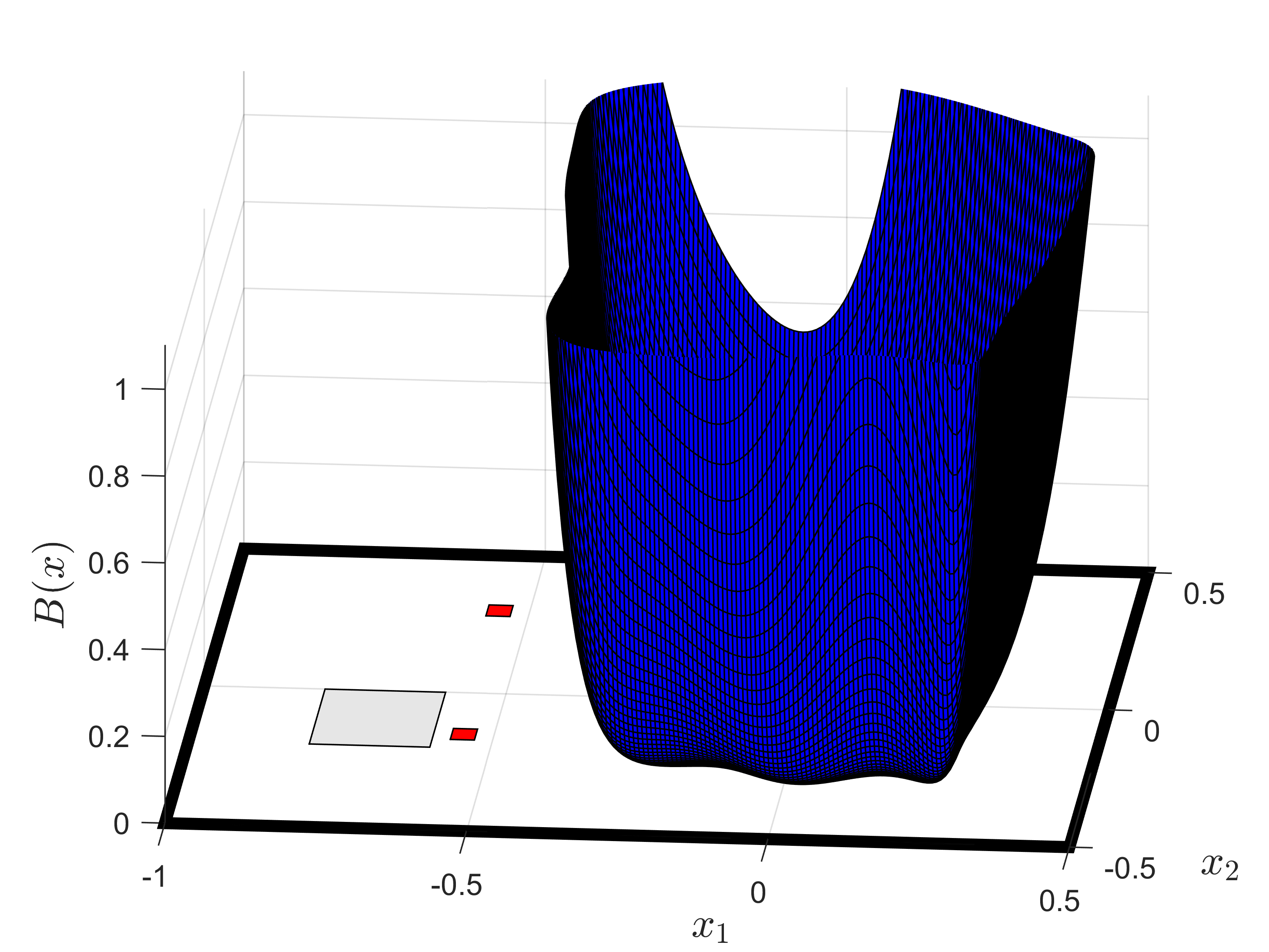}
        \subcaption{Degree-30 SOS SBF}
        \label{fig:barrier_sos}
        \end{subfigure}%
                \begin{subfigure}[t]{0.25\textwidth}
            \centering
            \includegraphics[width=1\linewidth]{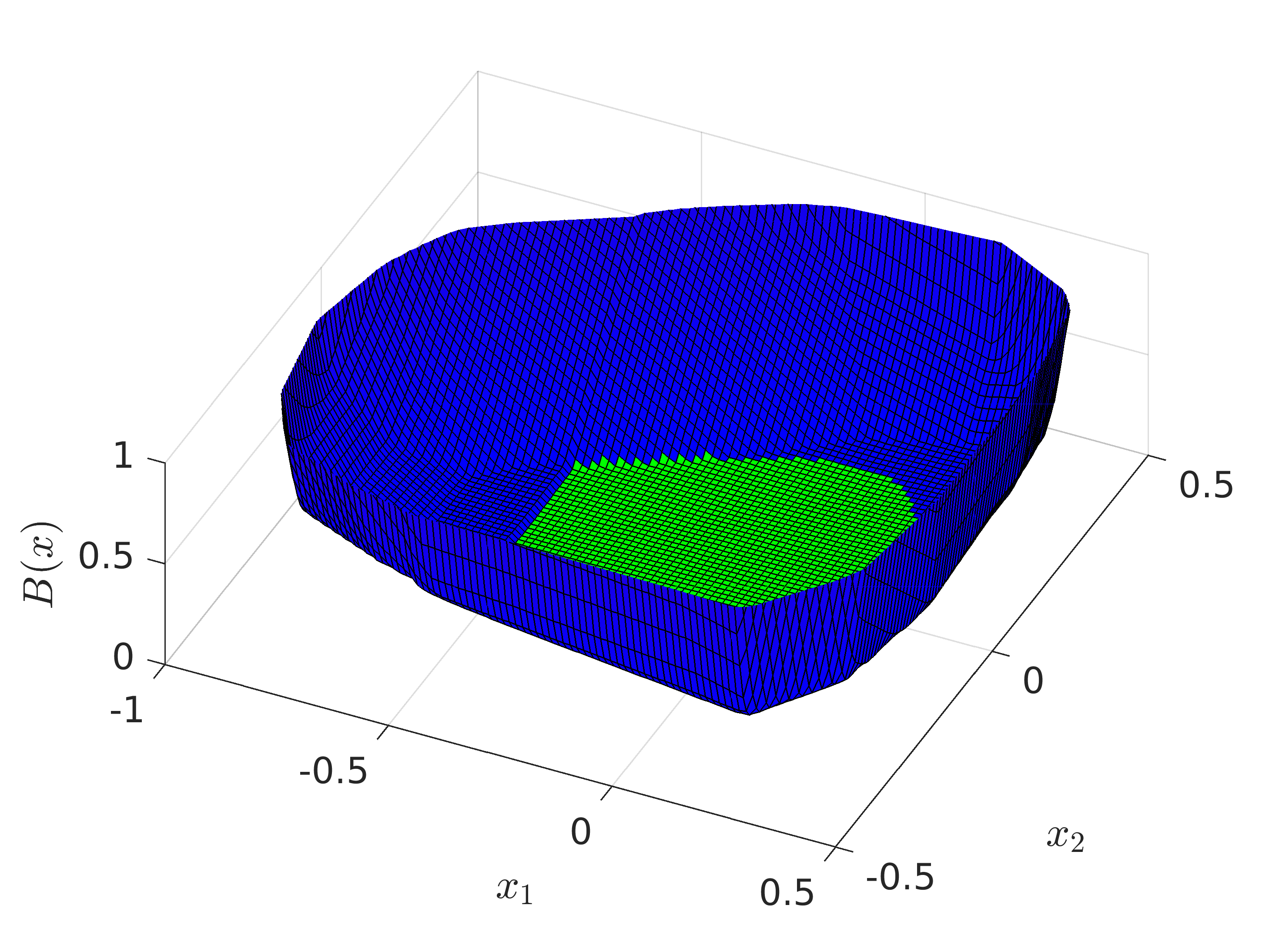}
            \subcaption{NBF}
                   \label{fig:nbf}
        \end{subfigure}%
        \begin{subfigure}[t]{0.25\textwidth}
            \centering
            \includegraphics[width=1\linewidth]{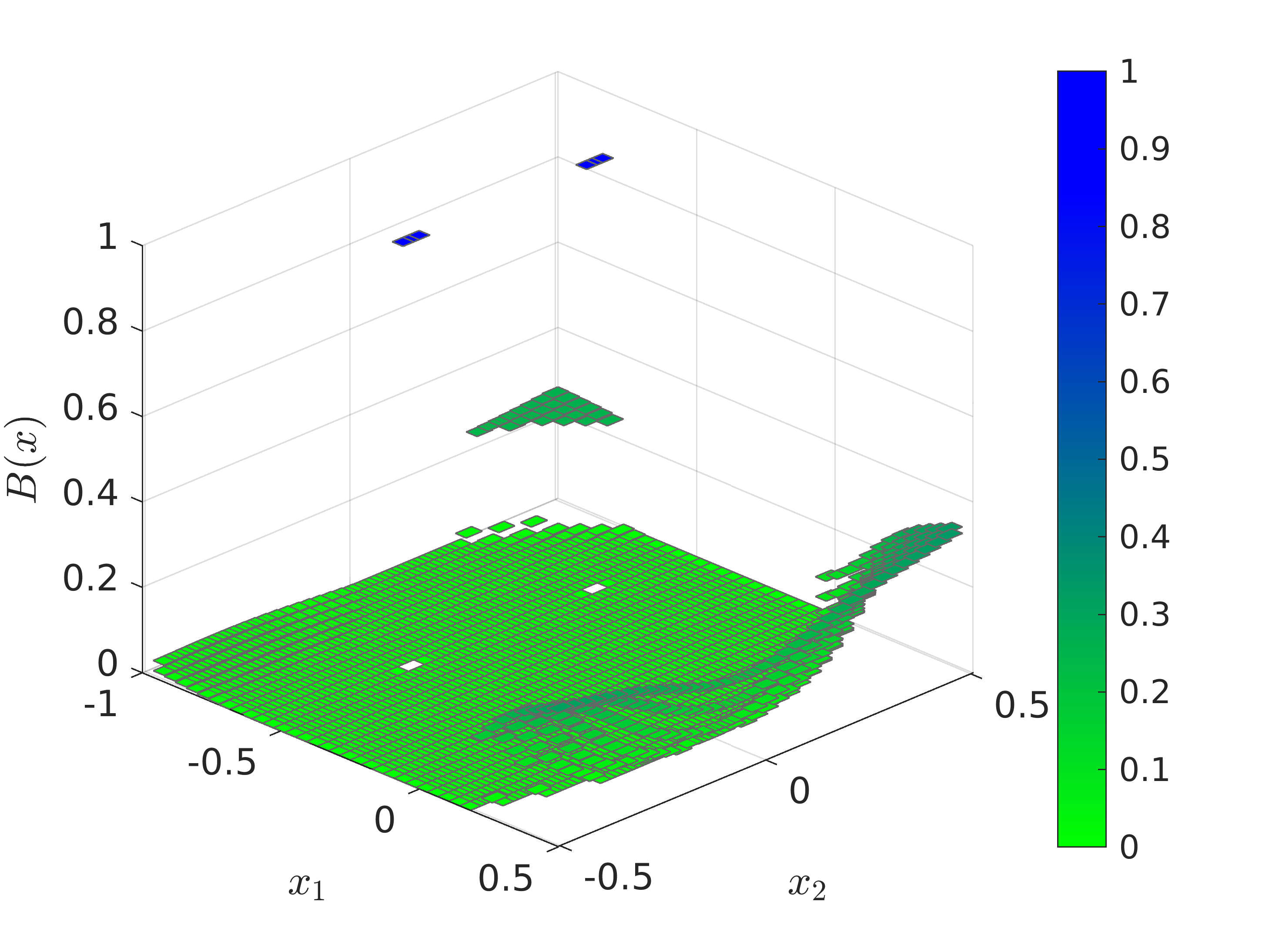}          
            \subcaption{PWC SBF}
            \label{fig:barrier_pwc}
        \end{subfigure}%
    \end{minipage}
    \caption{
    Example of a 2D stochastic system with a non-convex safe set and dynamics $\px_{k+1} = 0.5 \px_{k} + \pv$, where noise $\pv$ has a Gaussian distribution.  Fig. \ref{fig:env} shows the initial and safe sets. The SOS barrier of degree 30, the NBF, and the PWC-SBF for this system are shown in Figs. \ref{fig:barrier_sos} - \ref{fig:barrier_pwc}, respectively.  The obtained safety probability for 10 time steps and computation time for SOS are $P_{\text{SOS}} = 0.075$ and $\tau_{\text{SOS}} = 197$s, for NBF are  $P_{\text{NBF}} = 0.93$ and $\tau_{\text{NBF}} = 3600$s, and for PWC-SBF are $P_{\text{PWC}}= 0.93$ and $\tau_{\text{PWC}} = 69$s.
    }
    \label{fig:sos_pwc}
\end{figure*}

To illustrate these challenges, consider the 2-dimensional linear system with additive Gaussian noise in Figure~\ref{fig:sos_pwc}.  The initial and safe sets are depicted in Figure~\ref{fig:env}. Note that the safe set is non-convex.   Figures~\ref{fig:barrier_sos} and~\ref{fig:nbf} show the synthesized \gls{sos} SBF of degree 30 and \gls{nbf}, utilizing the state-of-the-art techniques presented in \cite{mazouz2022safety} and \cite{mathiesen2022safety}, respectively.  Generally, a higher value of the SBF indicates a lower safety probably bound for the system.  The corresponding lower bound on the safety probably (probability of remaining in the safe set) for 10 time steps is 0.075 for the \gls{sos} SBF and 0.93 for the \gls{nbf} with computation times of 197 seconds and 3600 seconds, respectively. Notice that \gls{nbf} outperforms \gls{sos} by more than $10$-fold in safety probability. However, it comes at the cost of an order of magnitude (nearly $20\times$) increase in the computation time for this 2-dimensional system. In fact, for \gls{nbf}, the computation time required to guarantee that a trained NN is an SBF quickly explodes as dimensionality increases.

In this manuscript, we introduce an SBF framework based on \emph{piecewise} (PW) functions. 
We first show a general formulation for PW-SBFs.  Then, we focus on PW-Constant (PWC) SBFs, whose simplicity gives rise to an elegant formulation applicable to \emph{general} stochastic systems.  Our key insight is that the challenging operations of expectation and function composition required by the martingale condition of SBFs can be significantly mitigated by using constant functions, enabling both simplicity and scalability.  Guided by this insight, we prove that the PWC-SBF synthesis problem reduces to a minimax optimization problem, for which we introduce three efficient computational methods. 
First, we show that, even for nonlinear stochastic systems, the minimax optimization problem can be reduced to a linear program (LP) via a dual formulation.  Second, to enable scalability, we introduce an iterative algorithm that requires solving two smaller LPs based on the \emph{counter-example guided synthesis} (CEGS) principle. 
CEGS however may need large memory.  
Third, we show that the minimax problem can be solved using a Gradient Descent (GD) algorithm, which admits even better scalability and resolves the memory issue of CEGS.  It however poses a challenge in designing a convergence criterion.
The power of PWC-SBFs is illustrated in Figure~\ref{fig:barrier_pwc}, where the same safety probability bound of $0.93$ as \gls{nbf} is achieved with $69$ seconds computation time, i.e., an order of magnitude ($50\times$) faster in computation. Our extensive evaluations demonstrate the efficacy of the proposed methods, outperforming state of the art and scaling to eight dimensional systems. 

In summary, the key contributions of this work are:
\begin{itemize}
    \item A general formulation for PW-SBFs with a particular focus on PWC-SBFs, providing a simple formulation for safety verification of general stochastic systems.
    \item A derivation of the PWC-SBF synthesis problem as a minimax convex optimization problem. 
    \item Three efficient and scalable computational methods for PWC-SBF synthesis:
    \begin{itemize}
        \item an LP duality-based approach with a proof of zero-duality gap, i.e., obtaining an exact solution to the minimax problem,
        \item a CEGS algorithm with a proof of convergence to the optimal solution in finite time, and
        \item a GD algorithm that offers significant reduction in computation time.
    \end{itemize}
    \item Extensive case studies to illustrate the performance of the proposed methods against state-of-the-art methods of \gls{sos} and \gls{nbf}, and an analysis hereof.
\end{itemize}
Overall, this work offers a significant advancement in SBF synthesis for safety-critical systems, providing both theoretical insights and practical computational methods that enhance reliability and scalability in complex environments.

\subsection{Related Work} 
Barrier functions, also known as barrier certificates~\cite{prajna2004safety,ames2014control, ames2016control}, offer a formal approach to ensure the safety of dynamical systems, particularly in nonlinear and hybrid models~\cite{prajna2004safety, prajna2006barrier}.  
Works~\cite{prajna2005necessity, ames2019control} establish 
the essential role of barrier certificates in providing safety guarantees for closed-loop systems.
The initial studies on barrier functions primarily focused on continuous-time and non-stochastic settings. 

Recent works~\cite{prajna2007framework, santoyo2021barrier, jagtap2020formal, mazouz2022safety, mathiesen2022safety} consider safety analysis of uncertain systems via stochastic barrier functions (SBFs).
Initially proposed in~\cite{kushner1967stochastic}, SBFs offer a framework for studying the probabilistic safety of stochastic systems grounded in martingale theory. Building on this foundation, \cite{prajna2007framework} introduces the concept of continuous-time SBFs, facilitating the construction of barrier certificates without the need for explicit computation of reachable sets. That method enables the handling of polynomial nonlinearities and uncertainty, ultimately providing an upper bound on the probability of unsafety.
Similar to those works, this paper focuses on probabilistic safety of stochastic systems, but assumes discrete-time dynamics.

Discrete-time stochastic system verification using SBFs was initially introduced in~\cite{santoyo2021barrier}, where the supermartingale condition was relaxed to a c-martingale for finite-time horizon problems. Similar to prior works, these barriers are computed through \gls{sos} optimization. 
While this method offers benefits in terms of problem convexity, it presents significant challenges in scalability and struggles with non-convex safe sets. Moreover, its applicability is confined to systems with polynomial dynamics or a subset thereof. 
In \cite{jagtap2020formal}, an alternative to \gls{sos} is proposed, leveraging counter-example guided inductive synthesis with SMT~\cite{de2008z3} or SAT~\cite{sebastiani2015optimathsat} solvers. These approaches draw from techniques for synthesizing Lyapunov functions~\cite{ravanbakhsh2015counter} and reachable sets~\cite{ravanbakhsh2017class} for switched systems using counter-examples. Those methods, however, cannot ensure termination within finite time.

To extend the \gls{sos} optimization technique to non-polynomial systems, \cite{mazouz2022safety} proposed a method by using tools from \cite{zhang2018efficient} to find piecewise affine bound on the dynamics. 
These bounds are utilized for the verification and control of neural network dynamic models. Works \cite{santoyo2021barrier,jagtap2020formal} propose iterative approaches to synthesize a control strategy to enforce safety of affine-in-control stochastic dynamical systems over finite time horizons. In \cite{mazouz2022safety}, an exact method for the synthesis of safe controllers is proposed by utilizing the \gls{sos} convexity property of the barrier. That still remains largely limited by the considered class of \gls{sos} functions. 

To address the limitations encountered by \gls{sos}, neural barrier functions (\gls{nbf}) \cite{mathiesen2022safety, dawson2022safe} have recently been proposed. 
\gls{nbf} use the representational power of neural networks (NNs) to train a highly amenable SBF candidate,
which then needs to be verified to establish whether it is a proper certificate.
Work \cite{vzikelic2023learning} 
combines \gls{nbf} with reinforcement learning to
train a deep policy based on a novel formulation of reach-avoid martingales. 
A major limitation of these methods is their reliance on NN verification, which suffer from scalability.
This paper aims to address the limitations of the \gls{sos} and \gls{nbf} approaches in terms of both conservatism and scalability.
Our work mainly builds on the results of \cite{laurenti2023unifying}, which establishes the relationship between barriers and dynamic programming. With this insight, we propose a novel certificate synthesis method based on piecewise functions. 

\section{Problem Formulation}
\label{sec:problem}


We consider a stochastic process described by the following stochastic difference equation
\begin{equation}
    \label{eq:system}
    \px_{k+1} = f(\px_k,\pv_k),
\end{equation}
where state $\pX_k$ takes values in $\reals^n$, and noise $\pv_k$ is an independent and identically distributed random variable taking values in $\mathbb{R}^{\mathrm{v}}$ with associated probability distribution $p_{\pv}$. Further, $f:\mathbb{R}^n \times \mathbb{R}^{\mathrm{v}} \to \mathbb{R}^n$ is the vector field representing the one step dynamics of System~\eqref{eq:system}. 
We assume that $f$ is almost everywhere continuous.
Intuitively,  $\px_{k}$ is a general model of a stochastic process with possibly nonlinear dynamics and non-additive noise. 

To define a probability measure for System~\eqref{eq:system}, for a measurable set $X$,
we define the one-step stochastic kernel $T(X\mid x)$ as
\begin{equation}
\label{eq:transition_kernel}
    T(X\mid x):= \int_{\mathbb{R}^{\mathrm{v}}} \mathbf{1}_X
(f(x,v))p_{\mathbf{v}}(dv),
\end{equation}
where $\mathbf{1}_{X}$ is the indicator function for set $X$, defined as  
\begin{equation*}
    \mathbf{1}_{X}(x) =
    \begin{cases}
        1 & \text{if } x \in X\\
        0 & \text{otherwise}
    \end{cases}.
\end{equation*}
It follows that for an initial condition $x_0 \in \mathbb{R}^n,$ $T$ induces a unique  probability measure  $\mathrm{Pr}^{x_0}$   \cite{bertsekas2004stochastic}, such that for measurable sets
$X_0,X_k \subseteq X$ where 
$k \in \naturals$ it holds that 
\begin{align*}
    &\mathrm{Pr}^{x_0}[\pX_0\in X_0] =  \mathbf{1}_{X_0}(x_0), \\
    &\mathrm{Pr}^{x_0}[\pX_k\in X_k   \mid \pX_{k-1}=x] =  T^{}(X_k \mid x).
\end{align*}
The definition of $\mathrm{Pr}^{x_0}$ allows one to make probabilistic statements over the trajectories of System \eqref{eq:system}. 
This work seeks to quantify \emph{probabilistic safety}, which together together with its dual, \emph{probabilistic
reachability} \cite{laurenti2023unifying}, are commonly used to quantify safety for stochastic dynamical systems. Work~\cite{abate2008probabilistic} presents a generalization of the notion of invariance.

\begin{definition}[Probabilistic Safety]
    Let $X_\safe \subset \mathbb{R}^n$ be a bounded set representing the safe set, $X_0 \subseteq X_\safe$ be the initial set, and $N \in \mathbb{N}$ be the time horizon. Then, \emph{probabilistic safety} is defined as
    $$ P_{\safe}(X_\safe,X_0,N)= \inf_{x_0 \in X_0} \mathrm{Pr}^{x_0}[\forall k \in \{0,...,N\}, \pX_k \in X_\safe]. $$
\end{definition}

The problem we consider in this work is as follows.

\begin{problem}[Safety Certificate]
\label{Prob:Verification}
Consider a safe set  $X_\safe \subset \reals^n$ and an initial set $X_0 \subseteq X_\safe$. Then, for a given threshold $\delta_\safe \in [0,1]$, certify whether, starting from $X_0$, System~\eqref{eq:system} remains in $X_\safe$ for $N$ time steps with at least probability $\delta_\safe$, i.e.,
$$P_\safe(X_\safe,X_0,N ) \geq \delta_\safe.$$
\end{problem}
Problem~\ref{Prob:Verification} seeks to compute the probability that $\pX_k$ remains within a given safe set, e.g., it avoids obstacles and undesirable states.
Computation of this probability is particularly challenging because System~\eqref{eq:system} is stochastic and $f$ can be a nonlinear (and non-polynomial) function. Further, since set $X_\safe$ can be non-convex, the problem is generally non-convex.  This implies that convex optimization tools cannot be used out-of-the-box for this problem in an efficient way.

\paragraph*{Approach overview}
Our approach to Problem~\ref{Prob:Verification} is based on SBFs, which are value functions to guarantee a lower bound on the probability of safety in the sense of Problem \ref{Prob:Verification}. We review the theory of SBFs in Section~\ref{sec:general_sbf_theory}.
To address the inefficiencies of existing synthesis techniques for continuous barrier functions due to, e.g., non-convexity of $X_\safe$, in Section \ref{sec:piecewise_sbf_theory}, we focus on \emph{piecewise} functions.
We show that, based on such functions, barrier certificates can be provided for the safety of System~\eqref{eq:system} by establishing a lower bound on $P_\safe$.  Further, we show in Section \ref{sec:pwc_sbf} that by choosing a piecewise constant function, the synthesis problem reduces to a linear optimization problem, for which we introduce three efficient computational frameworks. 
\section{Stochastic Barrier Functions}\label{sec:general_sbf_theory}

In this section, we provide an overview of stochastic barrier functions.

Consider System~\eqref{eq:system} with safe set $X_\safe \subset \mathbb{R}^n,$ initial set $X_0\subseteq X_\safe,$ and unsafe set $X_\unsafe=\mathbb{R}^n\setminus X_\safe$. Then, a function $B:\mathbb{R}^n \to \mathbb{R}$ is a \emph{stochastic barrier function} (SBF) if there exist scalars $\eta,\beta \geq 0$ such that
\begin{subequations}
    \begin{align}
        &B(x) \geq 0 \qquad &&\forall x\in \mathbb{R}^n\label{eq:barrier_ss},\\
        &B(x) \geq 1 \qquad &&\forall x\in X_\unsafe\label{eq:barrier_unsafe},\\
        &B(x) \leq \eta \qquad &&\forall x\in X_0\label{eq:barrier_initial},\\
        &\mathbb{E}[B(f(x, \pv))] \leq  B(x) + \beta \qquad && \forall x\in X_\safe \label{eq:barrier_expectation},
    \end{align}
\end{subequations}
where $\pv \in \mathbb{R}^n$ is a random variable with density $p_\pv$. 
Given an SBF $B$, then
a lower bound on the probability of safety for System~\eqref{eq:system} can be obtained according to the following proposition.
\begin{theorem}[{\cite[Theorem 2]{laurenti2023unifying}}]
\label{prop:barrier}
If there exists function $B$ that satisfies Conditions~\eqref{eq:barrier_ss}-\eqref{eq:barrier_expectation}, then for a horizon $N\in \naturals$, it follows that
    \begin{align}
        P_\safe(X_\safe,X_0,N )
        \geq 1-( \eta + \beta N). \label{eq:probabilityBarrierFunctions}
    \end{align}
\end{theorem}
Note that Equation~\eqref{eq:probabilityBarrierFunctions} provides a lower bound on the probability of remaining in the safe set. 
Hence, $B$ is called a barrier certificate
if the right hand side of Equation~\eqref{eq:probabilityBarrierFunctions} is greater than or equal to threshold $\delta_\safe$, i.e., 
$$1-(\eta + \beta N) \geq \delta_\safe.$$
A major benefit of using barrier certificates for safety analysis is that  static Conditions~\eqref{eq:barrier_ss}-\eqref{eq:barrier_expectation} enable probabilistic reasoning on the time evolution of the stochastic system without the need to evolve the system.
A popular approach to finding an SBF $B$ is to limit the search to the set of \gls{sos} polynomials of a given degree.  Then, a \gls{sos} optimization problem can be formulated to synthesize the parameters of the barrier polynomial, where the objective function is $\min \eta + N \beta$, subject to Conditions~\eqref{eq:barrier_ss}-\eqref{eq:barrier_expectation}
\cite{prajna2007framework, santoyo2021barrier, mazouz2022safety}. 
The \gls{sos} approach however becomes very conservative when $X_\safe$ is non-convex.
Recent techniques address this issue by using the power of neural networks as universal approximators to learn a \gls{nbf}~\cite{mathiesen2022safety, dawson2022safe}.
Nevertheless, while training \gls{nbf}s can be done efficiently, checking Conditions~\eqref{eq:barrier_ss}-\eqref{eq:barrier_expectation} against the network is challenging, which limits scalability of \gls{nbf}s.  
To tackle these challenges, we propose piecewise barrier functions below.





\section{Piecewise Stochastic Barrier Functions}\label{sec:piecewise_sbf_theory}

In this paper, we introduce the notion of \textit{piecewise} SBF (PW-SBF). 
Consider a partition $X_1,\ldots,X_K$ of safe set $X_\safe$ in $K$ compact sets, i.e.,
$$\bigcup_{i=1}^K X_i = X_\safe \quad \text{ and } \quad X_i\cap X_j = \emptyset,$$ for all $i\neq j \in \{1,\ldots,K\}$,
such that vector field $f$ is continuous in each region $X_i$ and we further assume that the boundary of each set has measure zero w.r.t. $T(\cdot  \mid x)$ for any $x \in X_s$.
Further, let $B_i: X_i \to \mathbb{R}$ be a real-valued continuous function for every $i \in \{1,\ldots, K\}$.  We define PW function 
\begin{equation}
    \label{eq:pwf}
    B(x) = 
    \begin{cases}
        B_i(x) & \text{if } x \in X_i\\
        1 & \text{otherwise.}
    \end{cases}
\end{equation}



The following corollary characterizes the general result of Theorem~\ref{prop:barrier} for the PW function $B$.

\begin{corollary}[Piecewise SBF]
    \label{th:pwa}
    Piecewise function $B(x)$ in Equation~\eqref{eq:pwf} is a stochastic barrier function for System~\eqref{eq:system} if $\forall i \in \{1,\ldots,K\}$ there exist scalars $\beta_i,\eta \geq 0$ such that
    \begin{subequations} 
        \begin{align}
            & B_i(x)\geq 0      \qquad \qquad \qquad \qquad \qquad \forall x \in X_i,
            \label{eq:pwb_cond1}
            \\ 
            & B_{i}(x) \leq \eta  \qquad \qquad  \qquad \qquad \qquad \forall x \in X_i \cap X_0,
            \label{eq:pwb_cond2}
            \\
            & \sum_{j = 1}^K \expect[B_j(\px') \mid \px' \in X_j] \cdot T(X_j \mid x) + T(X_\unsafe \mid x) \nonumber  \\
            & \qquad \;\: \leq  B_{i}(x) + \beta_{i} \qquad \qquad \qquad \forall x \in X_i,  \label{eq:total_law_exp}
        \end{align}
    \end{subequations}
    where $\px' = f(x,\pv)$ and $T$ is the transition kernel in Equation~\eqref{eq:transition_kernel}.  
    Then, it holds that 
    \begin{equation}
        \label{eq:pwb_safetyprob}
        P_\safe(X_\safe,X_0, N) \geq 1 - (\eta + N\cdot \max_{i \in \{1,...,K\} }\beta_i). 
    \end{equation}
\end{corollary}
\begin{proof}
    The proof builds on \autoref{prop:barrier}.
    It is clear that PW function $B$ in \eqref{eq:pwf} under Conditions~\eqref{eq:pwb_cond1}-\eqref{eq:pwb_cond2} satisfies Conditions~\eqref{eq:barrier_ss}-\eqref{eq:barrier_initial}.  Then, it is enough to show that 
    $\max_{i\in \{1,\ldots,K\} }\beta_i$ from Condition~\eqref{eq:total_law_exp}
    upper bounds $\expect[B(\px')]-B(x)$ in Condition~\eqref{eq:barrier_expectation}.
    For $x \in X_i$ and PW $B(x)$, we have
    \begin{align}
        \expect[B(&f(x, \pv))]-B(x) \label{eq:martingale_1} \\
        & = \sum_{j=1}^K \mathbb{E}[B_j(f(x, \pv))\mid f(x, \pv) \in X_j] \cdot T(X_j \mid x) \nonumber \\
        &\qquad \quad + 1 \cdot T( X_\unsafe \mid x) - B_i(x) 
        \label{eq:law_totoal_expetation}
        \\
        & \leq \beta_i,
        \label{eq:beta_i}
    \end{align}
    where the equality in \eqref{eq:law_totoal_expetation} holds by the law of total expectation, $B(x' \in X_\unsafe) = 1$, and the inequality in \eqref{eq:beta_i} holds by Condition~\eqref{eq:total_law_exp}.  Hence, for every $x \in X_\safe = \bigcup_{i=1}^K X_i,$ 
    the expression in \eqref{eq:martingale_1} is $\leq \max_{i\in\{1,...,K\}} \beta_i$.
\end{proof}

Corollary~\ref{th:pwa} enables one to formulate an optimization problem to synthesize a PW-SBF.  Specifically, the objective is to find $B_i$s that minimize ($\eta + N \cdot \max \{\beta_i\}_{i=1}^K$) subject to Conditions~\eqref{eq:pwb_cond1}-\eqref{eq:total_law_exp}.  The benefit of this formulation is that both the size of the partition and type of functions $B_i$s (e.g., linear, polynomial, exponential, etc.) are design parameters.
This provides  flexibility for the SBF to fit different shapes of $X_\safe$.  This flexibility, however, may introduce challenges as it can lead to a non-convex optimization problem, even for simple choices for $B_i$ such as  polynomial or linear functions. 
We introduce a set of simple but effective choices that lend themselves to efficient computational tools that outperform the state of the art SBF synthesis techniques.


A major difficulty in the optimization problem results from the product of the expectation term and the transition kernel function in Condition~\eqref{eq:total_law_exp}, namely,
$$\expect[B_j(\px') \mid \px' \in X_j] \cdot T(X_j \mid x).$$
The first term not only requires to perform an expectation operation, but also a composition of $B_j$ with the nonlinear function $f$.  The second term $T$ is a nonlinear function of $x$ that involves an integral of probability density function $p_\pv$.  
To reduce complexity, one can choose to use constant values for $B_i$s, which simply avoids the need to perform the expectation and composition operations.  This leads to a PW Constant (PWC) SBF.  
Furthermore, while the analytical form of $T$ may be hard to obtain, its bounds can be efficiently computed using, e.g., the procedure in \cite{skovbekk2023formal}, for general $f$ and non-standard $p_\pv$ (e.g., non-symmetric, non-unimodal, etc.).  In the next section, we detail the optimization problem for these choices.

\section{Piecewise Constant SBF Synthesis}\label{sec:pwc_sbf}

In this section, 
we first formally set up an optimization problem for synthesis of PWC-SBFs. 
Then, we introduce three efficient computational methods to solve the optimization problem, namely, a \gls{lp} duality-based approach, a \gls{cegs} procedure, and a \gls{gd}-based method.


\subsection{PWC-SBF Synthesis}

For $i,j \in \{1,\ldots,K\}$, let $\low{p}_{ij},\up{p}_{ij} \in [0,1]$ denote the lower and upper bounds of the transition kernel $T(X_j \mid x)$ for every $x \in X_i$, respectively, i.e., 
\begin{align}
    \label{eq:pij prob bounds}
    \quad \low{p}_{ij} \leq T(X_j \mid x) \leq \up{p}_{ij} \qquad \forall x \in X_i.
\end{align}
Similarly, we use $\low{p}_{i\unsafe},\up{p}_{i\unsafe} \in [0,1]$ for the bounds of $T$ to the unsafe set $X_\unsafe$, i.e.,
\begin{align}
    \label{eq:piu prob bounds}
    \qquad \quad \low{p}_{i\unsafe} \leq T(X_\unsafe \mid x) \leq \up{p}_{i\unsafe} \qquad \forall x \in X_i.
\end{align}
Note that these bounds can be computed efficiently for general $f$ and $p_\pv$ by using, e.g., techniques in
\cite{cauchi2019efficiency,
laurenti2020formal,
adams2022formal,skovbekk2023formal}.
We define the set of all feasible values for the transition kernel for all $x \in X_i$ as
\begin{align}
        \mathcal{P}_i = \Big\{ & p_i = (p_{i1},\ldots,p_{iK},  p_{i\unsafe})  \in [0,1]^{K+1} \quad s.t. \nonumber\\
        & \qquad \sum_{j=1}^K p_{ij}+p_{i\unsafe} = 1, \nonumber\\
        & \qquad \low{p}_{ij} \leq p_{ij} \leq \up{p}_{ij} \;\;\; \forall j \in \{1,\ldots,K,\unsafe \} \Big \}.
        \label{eq:feasible tran prob}
\end{align}

Note that $\mathcal{P}_i$ is a convex polytopic set, specifically a simplex.
The following theorem sets up an optimization problem for synthesis of a PWC-SBF.

\begin{theorem}[PWC-SBF Synthesis]
    \label{th:pwc}
    Given a $K$-partition of $X_\safe$, let $\mathcal{B}_K$ be the set of PWC functions in the form of Equation~\eqref{eq:pwf} with $B_i(x) = b_i \in \mathbb{R}$ for every $i \in \{1,\ldots,K\}$, and let $\mathcal{P} = \mathcal{P}_1 \times \ldots \times \mathcal{P}_K$, where each $\mathcal{P}_i$ is the set of probability distributions defined in Equation~\eqref{eq:feasible tran prob}.
    Then, $B^* \in \mathcal{B}_K$ is a PWC-SBF that maximizes RHS of Equation~\eqref{eq:pwb_safetyprob} if $B^*$ is a solution to the following optimization problem
    \begin{align}
        \label{eq:outer_opt}
        B^* = \arg\min_{B \in \mathcal{B}_K} \; \max_{(p_i)_{i=1}^K \in \mathcal{P}}
        \; \eta + N \beta 
    \end{align}
    subject to 
    \begin{subequations} 
        \begin{align}
            &b_i \geq 0  && \forall i \in \{1,\ldots,K\}, \label{eq:bj_nonnegative}\\
            &b_i \leq \eta &&  \forall i : X_i \cap X_0 \neq \emptyset, \label{eq:bj_initial}\\
            &\sum_{j = 1}^K b_j \cdot p_{ij} + p_{i\unsafe} \leq  b_i + \beta_{i} && \forall i\in \{1,\ldots,K\}, \label{eq:bj_martingale}\\
            &0 \leq \beta_{i} \leq \beta && \forall i\in \{1,\ldots,K\}. 
        \end{align}
    \end{subequations}
\end{theorem}
\begin{proof}
    It suffices to show that if for all $i$, $B_i(x)$ satisfy Conditions \eqref{eq:bj_nonnegative}-\eqref{eq:bj_martingale}, then PW function $B(x)$ as defined in Equation~\eqref{eq:pwf} satisfies Conditions \eqref{eq:pwb_cond1} - \eqref{eq:total_law_exp}. 
    The optimization problem in Equation \eqref{eq:outer_opt} aims to maximize the safety probability over feasible values for the transition kernel $\mathcal{P}_i$, which is expressed as a minimax problem. 
    Each $\beta_i$ is bounded according to Constraint \eqref{eq:bj_martingale}, and by Corollary \ref{th:pwa},  $\max_{i\in \{1,\ldots,K\} }\beta_i$  upper bounds $\expect[B(\px')]-B(x)$.  
    Hence, optimizing over objective $B^*$ also optimizes the right hand side of \eqref{eq:pwb_safetyprob}, thereby maximizing the lower bound on safety probability \( P_\safe(X_\safe, X_0, N) \). 
\end{proof}

Given a partition of $X_\safe$, Theorem~\ref{th:pwc} provides a method to synthesize a PWC-SBF that optimizes a lower bound on $P_\safe(X_\safe,X_0,N)$. 
This theorem immediately gives rise to questions on optimality (w.r.t. general SBFs) and computibility.  We address the computibility question in Section~\ref{sec:computation pwc-sbf}.  
For optimality,
the following proposition establishes that, in the limit of a partition of size large enough, a PWC-SBF converges to a value of safety smaller or equal to 
the optimal safety probability $\low{P}^*_\safe(X_\safe,X_0,N)$ obtainable with continuous SBFs.
\begin{proposition}
    \label{th:convergence}
    Let $\low{P}^*_\safe(X_\safe,X_0,N)$ denote the optimal safety probability obtained from Theorem~\ref{prop:barrier} from the class of continuous SBFs, i.e., the set of continuous functions satisfying Conditions~\eqref{eq:barrier_ss}-\eqref{eq:barrier_expectation}.
    Consider a uniform partition of $X_s$ in $K$ compact sets
    and call $(\eta+N\beta)_K^*$ the resulting optimal safety probability bound obtained from solving Theorem~\ref{th:pwc}. Then, it holds that
    \begin{equation*}
     \low{P}^*_\safe(X_\safe,X_0,N) \leq \lim_{K\to\infty} (1-(\eta+N\beta)_K^*).
    \end{equation*}    
\end{proposition}
\begin{proof}
    First of all, we need to show that for all $x\in X_u$, the optimal choice of the barrier is always $B(x)=1$. 
    In order to do that note that 
    \begin{multline*}
        \mathbb{E}[B(f(x, \pv))] = 
        \int_{v : f(x,v) \in X_\safe} B(f(x,v))p_{\pv}(dv) \, + \, \\
        \int_{v : f(x,v) \in X_\unsafe}B(f(x,v))p_{\pv}(dv).
    \end{multline*}
    Using this fact, we can rewrite Condition \eqref{eq:barrier_expectation}
    as the fact that $\forall x\in X_\safe$ it must hold that
    \begin{multline*}\
         \int_{v : f(x,v) \in X_\safe} B(f(x,v))p_{\pv}(dv) \, + \, \\
        \int_{v : f(x,v) \in X_\unsafe}B(f(x,v))p_{\pv}(dv)
        \leq  
        B(x) + \beta N .
    \end{multline*}
    As for $x\in X_u$, $B(x)$ only appear on the left hand side of the inequality; hence, the value of $\beta$ required to satisfy the inequality is minimized by taking the smallest possible value of $B(x)$ for $x\in X_u.$ Because of Condition~\eqref{eq:barrier_unsafe}, the smallest possible value is $1$. The rest of the proof follows by the fact that, within $X_s$, PWC-SBFs are dense w.r.t. to the set of continuous SBFs.
    %
\end{proof}


This proposition shows that, despite their simplicity, PWC-SBFs are as expressive as complex forms of SBFs such as polynomials and \gls{nbf}s.  That is, PWC-SBFs can compute the optimal safety probability as well as or even better than continuous SBFs.  In fact, our evaluations in Section~\ref{sec:evaluations} clearly demonstrate this point.
Further, we also note that, from the proof of Proposition~\ref{th:convergence}, it becomes evident that the choice of $B(x)=1$ for all $x\in X_\unsafe$ in Equation~\eqref{eq:pwf} is the optimal choice for PW-SBFs.

\subsection{Computation of PWC-SBFs}
\label{sec:computation pwc-sbf}
The optimization problem in Theorem~\ref{th:pwc} is a minimax problem with the decision variables
\begin{subequations}
    \begin{align}
        &b = (b_1,\ldots,b_K) \in \mathbb{R}^{K}_{\geq 0}, && \label{eq:b-variable}\\
        &p_i = (p_{i1}, \ldots, p_{iK}, p_{i\unsafe}) \in \mathcal{P}_i && \forall i\in \{1,\ldots,K\}, \label{eq:p-variable}\\
        & \beta_i \in \mathbb{R}_{\geq 0} && \forall i\in \{1,\ldots,K\} \label{eq:beta-variable},\\
        & \eta,\beta \in \mathbb{R}_{\geq 0}. &&
    \end{align}
\end{subequations}
A major difficulty in computability of this optimization problem is in Constraint~\eqref{eq:bj_martingale}, which includes products of decision variables $b_j$ and $p_{ij}$. 
As a consequence, the minimax optimization problem in Theorem \ref{th:pwc} is bilinear, meaning that it is linear if either $b_j$ or $p_{ij}$ are fixed. This class of problems is generally non-convex; hence, convex solvers, which provide efficiency and guaranteed convergence, cannot be utilized.
We propose three approaches to (losslessly) convexify the problem, namely dual Linear Programming (LP), Counter-Example Guided Synthesis (CEGS), and Gradient Decent (GD) such that the optimal solution can be efficiently computed.
Each approach has its own advantages and disadvantages. The dual LP approach allows use of standard LP solvers to exactly compute the optimal solution at the cost of lesser scalability. CEGS mitigates the scalability problem by decoupling the inner maximization and outer minimization, but may still require prohibitive amounts of memory. To address this problem, GD is presented, which utilizes sparsity and provides efficient gradient computation. However, GD requires tweaking of hyperparameters to achieve good convergence in practice, due to the non-smoothness of the objective function. We first present the dual LP approach.

\subsubsection{Dual Linear Program}
    \label{sec:dual lp}

Using duality, we first introduce an exact approach that encodes the optimization problem in Theorem~\ref{th:pwc} as a linear program. 
We begin by observing that the set of feasible transition kernels $\P_i$ in Equation~\eqref{eq:feasible tran prob} is a simplex.  Hence, it can be represented as
\begin{align}
    \P_i = \{p_i : H_i \, p_i \leq h_i \}, 
\end{align}
where matrix $H_i \in \reals^{2(K+1) \times (K+1)}$ and vector $h_i \in \reals^{2(K+1)}$ are defined by the constraints in Equation~\eqref{eq:feasible tran prob}, and inequality relation ``$\leq$'' is interpreted element-wise on vectors. 
%
%
%
Next, we define vector $\bar{b} = (b,1)$, where $b$ is in Equation~\eqref{eq:b-variable}, and dual variable $\lambda_i \in \mathbb{R}^{2(K+1)}_{\geq 0 }$.  Then, Constraint~\eqref{eq:bj_martingale}
can be re-written as two constraints:
\begin{subequations}
\label{eq:bj_martingale_dual}
    \begin{align}
        & h_i^\top \lambda_i \leq b_i + \beta_i, \\
        & H_i^\top \lambda_i = \bar{b}.
    \end{align}
\end{subequations}
Since both constraints are linear in the decision variables $\lambda_i$, $b$, and $\beta_i$, the optimization problem in Theorem~\ref{th:pwc} can be written as an LP, which we formalize in Theorem~\ref{th:pwc}.

\begin{theorem}[PWC-SBF Dual LP]
 \label{th:dual}
An optimal solution $(b^{*1}, \beta^{*1}, \eta^{*1}, \lambda_1^*, \ldots, \lambda_K^*)$ to the following LP coincides with an optimal solution $(b^{*2}, \beta^{*2}, \eta^{*2}, p_1^*, \ldots, p_K^*)$ to the optimization problem in Theorem~\ref{th:pwc} on the SBF decision variables $b, \beta, \eta$. That is, $(b^{*1}, \beta^{*1}, \eta^{*1}) = (b^{*2}, \beta^{*2}, \eta^{*2})$: 
\begin{equation*}
    \min \eta + \beta N
\end{equation*}
subject to
\begin{align*}
    & 0 \leq b_{i} \leq 1 && \forall i\in \{1,\ldots,K\},\\ 
    & b_{i} \leq \eta  &&\forall i : X_i \cap X_0 \neq \emptyset, \\   
    & h_i^\top \lambda_i \leq b_i + \beta_i && \forall i\in \{1,\ldots,K\},\\
    & H_i^\top \lambda_i = \bar{b}&& \forall i\in \{1,\ldots,K\},\\
    & \lambda_i \geq 0 && \forall i\in \{1,\ldots,K\},\\
    & 0 \leq \beta_{i} \leq \beta && \forall i\in \{1,\ldots,K\}. 
 \end{align*}

%
%
%
\begin{proof}
Theorem~\ref{th:dual} follows as an application of the following  lemma.

\begin{lemma}[Zero Duality Gap]
\label{lemma:duality_gap}
    Consider the following two optimization problems with decision variables $z = (b_1, \ldots, b_K, \beta_i)$ and $\lambda_i$ and feasible transition kernel set $ \P_i = \{p_i : H_i p_i \leq h_i \}$
    \begin{equation}
        \begin{aligned}
            \min_z & \quad \beta_i && \\
            \mathrm{s.t.} & \quad \bar{b}^{\top}p_i \leq b_i + \beta_i &&  \quad \forall p_i \in \P_i 
        \end{aligned}
        \label{eq:prob_robust_lp}
    \end{equation}
    and
    \begin{equation}
        \begin{aligned}
            \min_{z,\lambda_i} & \quad \beta_i \\
            \mathrm{s.t.} & \quad h_i^\top \lambda_i \leq b_i + \beta_i \\ 
            & \quad H_i^\top \lambda_i = \bar{b}, \\
            & \quad \lambda_i \geq 0.
        \end{aligned}
        \label{eq:prob_dual_lp}
    \end{equation}
     Let $z_1^*$ and $(z_2^*, \lambda_i^*)$ be optimal solutions to the Problems~\eqref{eq:prob_robust_lp} and \eqref{eq:prob_dual_lp}, respectively.
    Then, $z_1^* = z_2^*$ holds.
\end{lemma}
\begin{proof} 
    The constraint $\bar{b}^{T}p_i
    \leq b_i + \beta_i$, for all $p_i \in \P_i$ can be written as an inner optimization problem
    \[
        \left(\begin{aligned}
            \max_{p_i\in \P_i} & \quad \bar{b}^{T}p_i\\
            \mathrm{s.t.} & \quad H_i p_i \leq h_i
        \end{aligned}\right) \leq b_i + \beta_i.
    \]
    Due to strong duality of \gls{lp} \cite[Chapter 5]{boyd2004convex}, we can substitute the inner problem for its equivalent (asymmetric) dual problem
    \[
        \left(\begin{aligned}
            \min_{\lambda_i \geq 0} & \quad h_i^\top \lambda_i\\
            \mathrm{s.t.} & \quad H_i^\top \lambda_i = \bar{b}
        \end{aligned}\right) \leq b_i + \beta_i.
    \]
    Since the inner problem is a minimization, it can be lifted out into the outer problem to become the constraints in Problem~\eqref{eq:prob_dual_lp}.
\end{proof}

It now follows that by Lemma \ref{lemma:duality_gap}, the optimization problem in Theorem~\ref{th:dual} is an equivalent dual formulation, i.e., with zero duality gap, of the primal maximization problem in Theorem~\ref{th:pwc}. 
\end{proof}

\end{theorem}
Hence, $B(x)$ constitutes a proper stochastic barrier certificate, which guarantees probability of safety $P_\safe(X_\safe,X_0,N ) \geq 1 -  (\eta + \beta N)$.


\paragraph*{Computational Complexity}
The time complexity of a standard primal linear program is $\O(n^2m)$ where $n$ is the number of decision variables and $m$ is the number of constraints \cite{boyd2004convex}.
The program in Theorem \ref{th:dual} has $n = 2K^2 + 2K$ decision variables and $m = K^2 + 6K + L + 1$ where $L = \lvert \{i : X_i \cap X_0 \neq \emptyset \} \rvert$ is the number of regions intersecting with the initial set. Combining the complexity of a linear program with the number of decision variables and constraints for our dual linear program, we get $\O(n^2m) = \O(K^6)$. For many applications this may be prohibitive, hence, in the following subsections, we introduce two iterative approaches that lead to an improved runtime complexity.

\subsubsection{Counter-Example Guided Synthesis}
    \label{sec:cegs}



Here, we introduce another PWC-SBF synthesis method that is exact and computationally more efficient than the dual LP approach. 
The method is based on splitting the minimax problem in Theorem~\ref{th:pwc} into two separate LPs. 
One LP generates a candidate PWC-SBF that optimizes the unsafety probability given a set of finite feasible distributions $\tilde{\P}_i \subset \P_i$, and the other LP generates distribution witnesses (counter-examples) that violate the safety probability guarantee of the candidate PWC-SBF.  Then, the witnesses are added to $\tilde{\P}_i$, and the process repeats until no more counter-example can be generated.  We dub this method as \acrfull{cegs} for PWC-SBF.



\algrenewcommand\algorithmicrequire{\textbf{Input:}}
\algrenewcommand\algorithmicensure{\textbf{Output:}}

\begin{algorithm}[t]
    \caption{CEGS for PWC-SBF}
    \label{alg:cegs}
    \begin{algorithmic}
        \Require Initial set $X_0$, partition $X = \{X_i\}_{i=1}^K$, time horizon $N$, and feasible transition kernel sets $\{\P_i\}_{i=1}^K$. 
        \Ensure Optimal PWC-SBF $B^*$ and safety probability bound $\low{P}^*_\safe$.
        \Statex \hrulefill
        \State $\{\tilde{\P}_i \gets \textsc{sampleDist}(\P_i)\}_{i=1}^K$ \Comment{Initialize $\tilde{\P}_i$}
        \State $\beta^* \gets 0$  \Comment{Initialize $\beta^*$}
        \State $\{\beta_i \gets 1\}_{i=1}^K$ \Comment{Initialize $\beta_i$}
        \While{$\beta^*  < \max\beta_i$}
        \State $b^*, \eta^*, \beta^* \gets  \textsc{pwbSynthLP}(X_0,N,X, \{\tilde{\P}_i\}_{i=1}^K)$
        \For{$i \gets 1$ to $K$}
            \State $p_i, \beta_i \gets \textsc{counterExDistLP}(b^*,\P_i, i)$
            \State $\tilde{\P}_i \gets \tilde{\P}_i \cup \{p_i\}$   \Comment{add counterexamples}
        \EndFor
        \EndWhile
        \State \Return $B^* \gets (b^*,1)$ and $\low{P}^*_\safe \gets 1 - (\eta^* + N\beta^*)$
    \end{algorithmic}
\end{algorithm}

The \gls{cegs} algorithm is shown in Alg.~\ref{alg:cegs}, which relies on subroutines in Algs.~\ref{alg:pwb-lp} and \ref{alg:counterexample-lp}. 
The main algorithm first sets up the counter-example distribution sets $\tilde{\P}_i$ by sampling a feasible distribution from each $\P_i$.
Then, based on $\tilde \P_i$, it synthesizes a candidate optimal vector $b^*$ that minimizes the unsafety probability with its corresponding scalar $\beta^*$ using the subroutine \textsc{pwbSynthLP}.  As shown in Alg.~\ref{alg:pwb-lp}, this subroutine is an LP that captures the $\min$ component (outer optimization) of the minimax problem in Theorem~\ref{th:pwa}.  

Next, given $b^*$, for each partition region $X_i$, an optimal distribution $p_i$ that maximizes its corresponding martingale gap $\beta_i$ is computed using subroutine \textsc{counterExDistLP} and then added to the set of witnesses $\tilde{\P}_i$.  As shown in Alg.~\ref{alg:counterexample-lp}, \textsc{counterExDistLP} is also an LP, and it captures the $\max$ component (inner optimization) of the minimax problem.  
Note that the function $\mathrm{Vertices}(\P_i)$ returns the (finite) set of vertices of simplex $\P_i$. 
If the obtained martingale gap $\beta_i$ is greater than $\beta^*$ for some region $X_i$, it means that the candidate $b^*$ is not optimal and there exists at least a distribution that violates the probabilistic guarantee of the candidate. Hence, the process repeats with the updated witnesses until $\beta_i \leq \beta^*$ for all $i \in \{1,\ldots,K\}$.

\begin{algorithm}[t]
    \caption{$\textsc{pwbSynthLP}(X_0,N,X,\{\tilde{P}_i\}_{i=1}^K)$}
    \label{alg:pwb-lp}
    \begin{algorithmic}
        \Require Initial set $X_0$, partition $X = \{X_i\}_{i=1}^K$, time horizon $N$, and finite set of distributions $\{\tilde{\P}_i\}_{i=1}^K$. 
        \Ensure Optimal vector $b^* = (b_1,\dots,b_K)^*$ and scalars $\eta^*$ and $\beta^*$ that minimize unsafety probability w.r.t. the given finite distribution sets $\tilde{P}_1, \ldots, \tilde{P}_K$.
        \Statex \hrulefill
    \State $b^*,\beta^*,\eta^* \gets \arg\min_{b,\beta,\eta} \eta + N\cdot \beta$ 
        \begin{align*}
            \quad &\text{subject to:}  \\
            & \qquad 0 \leq b_{i} \leq 1 && \quad  \forall i \in \{1,\ldots,K\} \\ 
            & \qquad b_{i} \leq \eta  && \quad \forall i : X_i \cap X_{0} \neq \emptyset \\
            & \qquad \sum_{j=1}^{K} b_j \cdot p_{ij} + p_{i\unsafe} \leq b_i + \beta && \quad \forall i \text{ and } \forall p_i \in \tilde{\P}_i
        \end{align*}
    \State \Return $b^*,\eta^*,\beta^*$
    \end{algorithmic}
\end{algorithm}

\begin{algorithm}[t]
    \caption{$\textsc{counterExDistLP}(b,\P_i,i)$}
    \label{alg:counterexample-lp}
    \begin{algorithmic}
        \Require Vector of constants $b = (b_1,\ldots,b_k)$, a set of feasible distributions $\P_i$ as defined in Equation~\eqref{eq:feasible tran prob}, and region index $i$.
        \Ensure Counterexample distribution $p_i = (p_{i1},\ldots,p_{iK},p_{i\unsafe})$ that maximizes martingale gap $\beta_i$ with respect to $b$.
        \Statex \hrulefill
    \State $p_i,\beta_i \gets \arg\max_{p_i,\beta_i} \beta_i$ 
        \begin{align*}
            \quad &\text{subject to:}  \\
            & \qquad \sum_{j=1}^{K} b_j \cdot {p}_{ij} + p_{i\unsafe} = b_i + \beta_i && \\ 
            & \qquad p_i \in \mathrm{Vertices}(\mathcal{P}_i)
        \end{align*}
    \State \Return $p_i,\beta_i$
    \end{algorithmic}
\end{algorithm}

The following theorem guarantees that the \gls{cegs} algorithm terminates in finite time with an optimal solution.

\begin{theorem}[\gls{cegs} for PWC-SBF]
\label{th:cegs}
   Alg.~\ref{alg:cegs} terminates in finite time with PWC-SBF $B^*$ that is an optimal solution to the problem in Theorem~\ref{th:pwc}.
\end{theorem}
\begin{proof}
    \label{prop:iterative}
The strategy of the proof is as follows: (i) we first prove convergence in finite number of iterations of Alg.~\ref{alg:cegs} using standard results from linear programming and then, (ii) we prove that the algorithm converges to the optimum using a proof by contradiction.

To prove convergence of Alg.~\ref{alg:cegs} in finite number of iteration, we 
start by observing that the constraint $\sum_{j=1}^{K} b_j \cdot p_{ij} + p_{i\unsafe} \leq b_i + \beta_i$ for all $p_i \in \mathcal{P}_i$ is true if and only if $\sum_{j=1}^{K} b_j \cdot p_{ij} + p_{i\unsafe} \leq b_i + \beta_i$ for all $p_i \in \mathrm{Vertices}(\mathcal{P}_i)$ \cite{ben1998robust}.
Hence for the $\beta_i$ computed using Alg.~\ref{alg:counterexample-lp}, it holds that $\sum_{j=1}^{K} b_j \cdot p_{ij} + p_{i\unsafe} \leq b_i + \beta_i$ for all $p_i \in \mathcal{P}_i$.
Next, let $(b^*, \eta^*, \beta^*)$ denote the optimal solution of Alg.~\ref{alg:pwb-lp} with respect to a given set of distributions $\{\tilde{\mathcal{P}}_i\}_{i=1}^K$.
Notice that $\beta^*$ is a (non-strict) lower bound for $\beta_i$ for all $i \in \{1, \ldots, K\}$, provided that $\beta_i$ is computed from $b^*$, since $\tilde{\mathcal{P}}_i \subset \mathcal{P}_i$. Then, there are two cases:
\begin{itemize}
    \item If for some $i \in \{1, \ldots, K\}$ it holds that $\beta^* < \beta_i$, then there exists a distribution $p_i \in \mathcal{P}_i \setminus \tilde{\mathcal{P}}_i$ such that $\sum_{j=1}^{K} b^*_j \cdot p_{ij} + p_{i\unsafe} > b^*_i + \beta^*$. 
    As a consequence, if Alg.~\ref{alg:counterexample-lp} returns a $\beta_i$ such that $\beta^* < \beta_i$ then the accompanying distribution $p_i\in \mathrm{Vertices}(\mathcal{P}_i)$ is a \emph{previously unseen} counter-example for the condition $\sum_{j=1}^{K} b^*_j \cdot p_{ij} + p_{i\unsafe} \leq b^*_i + \beta^*$ for all $p_i \in \mathcal{P}_i$.
    \item If instead $\beta^* = \max \beta_i$, then it holds that $\sum_{j=1}^{K} b^*_j \cdot p_{ij} + p_{i\unsafe} \leq b^*_i + \beta^*$ for all $p_i \in \mathcal{P}_i$.
\end{itemize}
As a result, to conclude it is enough to note that since $\mathcal{P}_i$ is a simplex, it has a finite number of vertices.



To prove that  Alg.~\ref{alg:cegs} converges to the optimal solution, we can employ a proof by contradiction. Assume that  $\beta^* = \max \beta_i$ and the resulting $\beta^*$ is not optimal. Then, there 
must exists a distribution $p_i \in \mathcal{P}_i$ such that $\textsc{pwbSynthLP}(X_0,N,X, \{\tilde{\P}_i\}_{i=1}^K \cup p)$ returns a smaller $\beta^*$ compared to $\textsc{pwbSynthLP}(X_0,N,X, \{\tilde{\P}_i\}_{i=1}^K )$, but this is a contradiction as $  \{\tilde{\P}_i\}_{i=1}^K \subseteq \{\tilde{\P}_i\}_{i=1}^K \cup p $.
%
\end{proof}

\paragraph*{Computational Complexity}
We note that the LP in subroutine \textsc{counterExDist} (Agl.~\ref{alg:counterexample-lp}) can be solved more efficiently using a procedure called O-maximization as proposed in \cite{givan2000bounded}.  The time complexity of this procedure is $\O(K \log K)$, whereas a standard LP algorithm runs in $\O(K^3)$ time. 
In the worst-case, all the vertices of the simplex $\P_i$ are explored for all $i \in \{1,...,K\}$, which has time complexiy $\O(K^2)$.
It is further noted that Alg. \ref{alg:cegs} keeps track of the history of counter-examples $p_i$s added at each iteration, making the approach potentially memory intensive. 

\subsubsection{Projected Gradient Descent}
    \label{sec:gd}

While \gls{cegs} is faster than the dual approach, it can be memory intensive. To alleviate this and allow better scalability, we present a third method for computing a PWC-SBF, namely, a gradient descent-based approach. 

With an abuse of notation, let $\beta_i(b)$ be martingale gap of region $X_i$ for a given PWC-SBF defined by vector $b = (b_1,\ldots,b_K)$.  Specifically,
\begin{equation}
    \label{eq:beta function}
    \beta_i(b) = \sup_{p_i \in \P_i} \max \Big\{0, \; \sum_{j=1}^{K} b_j \cdot {p}_{ij} + p_{i\unsafe} - b_i \Big\},
\end{equation}
and $\beta(b) = \max_i \beta_i(b)$. Similarly, we denote 
\begin{equation}
    \label{eq:eta function}
    \eta(b) =  \max_{i: X_i \cap X_0 \neq \emptyset} b_i.    
\end{equation}
Then, we define loss (objective) function 
\begin{equation}
    \label{eq:loss function}
    \mathcal{L}(b) = \eta(b) + N \cdot \beta(b).
\end{equation}
This loss function indeed describes the objective function of the minimax problem in Theorem~\ref{th:pwc}.  Hence, by minimizing $\L(b)$, we solve the minimax problem. Below, we show that $\L(b)$ is convex, and hence, we can use a gradient descent-based method for its optimization.  
More precisely, our approach is based on projected subgradient descent, since the elements of $b$ are constrained to $[0,1]$ and the maximization and supremum in Equation~\eqref{eq:beta function} are not smooth, admitting only subgradients.


\begin{theorem}[Convexity of $\L(b)$]
    \label{th: convex loss}
    The objective function $\mathcal{L}(b)$ in Equation~\eqref{eq:loss function} is convex in $b$.
\end{theorem}
\begin{proof}
    The proof follows a standard structure from disciplined convex programming, where functions are composed under convexity-preserving operations. We start by proving that $\beta_i(b)$ is convex in $b$. To this end, observe that $\sum_{j=1}^{K} b_j \cdot {p}_{ij} + p_{i\unsafe} - b_j$ is convex in $b$, invariant to the value $p_i$. Next, $\max(0, \cdot)$ is a convexity-preserving function, thus $\beta_i(b)$ is convex in $b$. The finite maximization for both $\eta$ and $\beta$ are once again convexity-preserving operations and finally, addition is convexity-preserving. This concludes the proof that $\mathcal{L}(b)$ is convex in $b$.
\end{proof}

Notice that the computation of $\beta_i(b)$ in Equation~\eqref{eq:beta function} is equivalent to the inner optimization problem of \gls{cegs}.  Hence, the  \textsc{counterExDistLP} routine in Alg.~\ref{alg:counterexample-lp} (based on O-maximization method) can be used for efficient computation of subgradients of $\beta_i$.
One subgradient for $\mathcal{L}(b)$ is the following
\begin{equation}
    \nabla_b \mathcal{L}(b) = \vec{1}_\eta(b) + N \cdot \vec{1}_\beta(b),
\end{equation}
where $\vec{1}_\eta(b) = \vec{1}_i$ is a one-hot vector with the 1 being in element $i = \mathrm{argmax}_{i: X_i \cap X_0 \neq \emptyset} b_i$, and 
\begin{equation}
    \vec{1}_\beta(b) = \nabla_b \beta_l(b) = p_l^* - \vec{1}_l,
\end{equation}
where $l = \arg\max_l \beta_l(b)$, and $p_l^* = \arg\max_{p_l \in \P_l} \sum_{j=1}^{K} b_j \cdot {p}_{lj} + p_{l\unsafe} - b_l$.

A challenge with applying subgradient methods in practice is that small step sizes are required due to the non-smoothness, slowing down the convergence. 
We propose to ameliorate this by substituting the maximization in $\mathcal{L}(b)$ with an $L_p$-norm where $1 < p < \infty$. 
Specifically, let $\tilde{\eta}(b) = (b_{i_1}, \ldots, b_{i_m})$ be a vector of PWC-SBF values corresponding to the regions that overlap with $X_0$, and $\tilde{\beta}(b) = (\beta_1(b), \ldots, \beta_K(b))$ be a vector of $\beta(b)_i$ values. Then, the proposed loss function is
\begin{equation}
    \label{eq: improved loss}
    \tilde{\mathcal{L}}(b) = \lVert \tilde{\eta}(b) \rVert_p + N \lVert \tilde{\beta}(b) \rVert_p,    
\end{equation}
and the corresponding gradient is
\begin{equation}
    \label{eq: improved loss gradient}
    \begin{aligned}
        \nabla_b \tilde{\mathcal{L}}(b) &= \sum_{k = 1}^m \left(\frac{\tilde{\eta}(b)_k}{\lVert \tilde{\eta}(b) \rVert_p}\right)^{p-1} \vec{1}_{i_k} + \\
        &\phantom{=}\;\sum_{i = 1}^K \left(\frac{\tilde{\beta}_i(b)}{\lVert \tilde{\beta}(b) \rVert_p}\right)^{p-1} \nabla_b \beta_i(b).
    \end{aligned}
\end{equation}

The benefit of this loss is that $L_p$-norms are smooth, and that $\tilde{\L}(b) \geq \L(b)$ for every $p_i \in \P_i$ and all $b$. Furthermore, the relative magnitude is bounded, $\lVert y \rVert_p \,/\, \lVert y \rVert_\infty \leq \sqrt[\leftroot{-1}\uproot{3}p]{r}$ for any $y \in \mathbb{R}^r$.
With the $L_p$-norm, we may tune the over-approximation to a trade-off between smoothness and tighter approximating the $L_\infty$-norm. 
The modified loss, through smoothness, also addresses the issue that each $\P_i$ often is sparse, i.e., $\up{p}_{ij} \approx 0$ for many $j$. Therefore, $\nabla_b \mathcal{L}(b)$ is often 0 for most regions. 
We remark that the modified loss $\tilde{\mathcal{L}}(b)$ is not smooth due to the maximization and supremum in the computation of $\beta_i(b)$. 

The projected subgradient descent procedure is described in Alg. \ref{alg:gd}. For each iteration, the algorithm calculates the gradient according to Equation~\eqref{eq: improved loss gradient} and an appropriate decreasing step size $\alpha$, followed by the gradient step. Since the PWC-SBF must reside in $[0, 1]^K$ and the gradient step may violate this, it is projected back onto the admissible space. Finally, since subgradient methods are not a steepest descent method, the algorithm keeps track of the best observed loss.

\begin{algorithm}[t]
    \caption{GD for PWC-SBF}
    \label{alg:gd}
    \begin{algorithmic}
        \Require Initial set $X_0$, partition $X = \{X_i\}_{i=1}^K$, time horizon $N$, and feasible transition kernel sets $\{\P_i\}_{i=1}^K$. 
        \Ensure Optimal PWC-SBF $B^*$ and safety probability bound $\low{P}^*_\safe$
        \Statex \hrulefill
        \State $b^{(1)} \gets (p_{1\unsafe}, \ldots, p_{K\unsafe})$  \Comment{Initialize $b$}
        \State $b_{\mathrm{best}} \gets \Call{copy}{b}$ \Comment{Initialize $b_{\mathrm{best}}$}
        \State $\mathcal{L}_{\mathrm{best}} \gets \infty$
        \State $k \gets 1$ \Comment{Iteration number}
        \While{not converged}
            \State $g \gets \nabla_{b^{(k - 1)}} \tilde{\mathcal{L}}(b^{(k - 1)})$ \Comment{According to \eqref{eq: improved loss gradient}}
            \State $\alpha_k \gets \Call{StepSize}{k, g}$ 
            \State $b^{(k)} \gets b^{(k - 1)} - \alpha_k g$ \Comment{Gradient descent step} 
            \For{$i \gets 1$ to $K$} \Comment{Project onto $[0, 1]$}
                \State $b^{(k)}_i \gets \max(0, \min(1, b^{(k)}_i))$
            \EndFor
            \State $\eta_{\mathrm{best}}, \beta_{\mathrm{best}} \gets \lVert \tilde{\eta}(b^{(k)}) \rVert_\infty, \lVert \tilde{\beta}(b^{(k)}) \rVert_\infty$
            \If{$\eta_{\mathrm{best}} + N\beta_{\mathrm{best}} < \mathcal{L}_{\mathrm{best}}$}
                \State $b_{\mathrm{best}}  \gets \Call{copy}{b^{(k)}}$
                \State $\mathcal{L}_{\mathrm{best}} \gets \eta_{\mathrm{best}} + N\beta_{\mathrm{best}}$
            \EndIf
            
            \State $k \gets k + 1$
        \EndWhile
        
        \State \Return $B^* \gets (b_{\mathrm{best}},1)$ and $\low{P}^*_\safe \gets 1 - \mathcal{L}_{\mathrm{best}}$
    \end{algorithmic}
\end{algorithm}

To guarantee convergence of gradient descent (Alg.~\ref{alg:gd}) towards the optimum, it is paramount that the objective is convex. 
Theorem~\ref{th: convex loss} shows that $\L(b)$ is convex. From this,
it immediately follows that $\tilde{\L}(b)$ is also convex.
\begin{corollary}
    The modified objective $\tilde{\mathcal{L}}(b)$ is convex in $b$.
\end{corollary}
\begin{proof}
    The proof follows from the fact that $\eta$ and $\beta$ are convex in $b$. Further, $L_p$-norms are convexity-preserving operations \cite{boyd2004convex}.
\end{proof}

We may characterize the non-asymptotic convergence of Alg.~\ref{alg:gd}, i.e. that it finds an $\epsilon$-optimal $b$ within a finite number of steps. 
\begin{proposition}
    Let $\epsilon$ denote the desired error, the gradient bound $G \geq \sup_{b \in [0, 1]^K} \lVert \nabla_b \tilde{\mathcal{L}}(b) \rVert_2$, and the upper bound of the initial distance to the optimal barrier $b^*$ by $R \geq \lVert b^{(1)} - b^* \rVert_2$. Then there exist integers $N_1$ and $N_2$ such that $\alpha_k \leq \epsilon / G$ for all $k > N_1$ and
    \[
        \sum_{k = 1}^{N_2} \alpha_k \geq \frac{1}{\epsilon}\left(R^2 + G^2 \sum_{k = 1}^{N_1} \alpha_k^2 \right).
    \]
    Then for any $k > N = \max(N_1, N_2)$ it holds that 
    \begin{equation}
        \tilde{\mathcal{L}}(b_{\mathrm{best}}) - \tilde{\mathcal{L}}(b^*) \leq \epsilon.
    \end{equation}
\end{proposition}
The $\epsilon$-optimal convergence follows directly from a result for convergence with strictly decreasing and non-summable step sizes in \cite{boyd2003subgradient}.
A conservative bound on $R$ is $K$ and on $G$ is $m + NK \sqrt{2}$.
To conclude, Alg.~\ref{alg:gd} converges to a desired precision in a finite number of steps \cite{boyd2003subgradient}. 

\paragraph*{Computational Complexity}
The complexity per iteration of the projected subgradient descent (Alg.~\ref{alg:gd}), utilizing the O-maximization procedure to compute $\beta_i$, is $\O(K \log K)$. We emphasize that much of the computation can be parallelized.

\section{Evaluations}
\label{sec:evaluations}

\begin{table*}[t]
\caption{Benchmark results for the \emph{piecewise constant} and \emph{continuous} SBF methods. The three PWC-SBF synthesis methods are:  dual LP, \gls{cegs} and \gls{gd}. These are compared against state-of-the-art methods: \gls{sos} and \gls{nbf}. 
$K$ denotes the size of partition of $X_\safe$, 
$\tau_{p}$ is the computation time for the bounds on the transition kernel,
$\low{P}_\safe$ is the obtained lower bound on the safety probability for $N=10$ time steps,
$\tau_o$ is the computation time for the SBF synthesis,
and \textrm{Deg} is the degree of the \gls{sos} polynomial for SBF.
}
\label{table:results}
\resizebox{1\textwidth}{!}{%
\begin{tabular}{@{}c|c c| cc | cc | cc | cc| ccc 
!{\color{white}\vline} @{}}
\toprule
 &  &  &   \multicolumn{5}{c }{\hspace{6mm} \underline{Piecewise Constant Stochastic Barriers}} 
  &  &    \multicolumn{4}{c }{\hspace{10mm}\underline{Continuous Stochastic Barriers$^{*}$}} \\
 & & & \multicolumn{2}{c|}{Dual LP} & \multicolumn{2}{c|}{CEGS} &  \multicolumn{2}{c|}{GD} &  \multicolumn{2}{c|}{NBF} & \multicolumn{3}{c}{SOS} \\
 \cline{2-14}
  {Model}   & {$K$} & $\tau_{p} (s) $  & $\low{P}_\safe$ &$\tau_{o}(s)$ 
   & $\low{P}_\safe$   & $\tau_{o} (s) $ & $\low{P}_\safe$   & $\tau_{o} (s) $ & $\low{P}_\safe$ & $\tau_{o} (s)$ & \textrm{Deg} & $\low{P}_\safe$ & $\tau_{o} (s) $ \\
  \hline 
  {Linear}  &  64 & 0.02  & 0.992 &  0.52
 &   0.992   &  0.04 &  0.952 &   0.04   & 0.585 & 3850.93 &4 &  0.582 & 0.014  \\
 2D &    225 & 0.31  & 0.998   & 164.60  &  0.998 &   0.44  & 0.973  & 0.20  & 0.940 & 3991.47 &  8& 0.582 & 0.265  
 \\ 
  \textit{Convex }&  900  &  8.85 & 0.999 & 1087.78    & 0.999   & 17.93  & 0.990  & 7.22 & 0.961& 4025.67 & 30  & 0.978  & 151.16  \\
  &  2500  &   41.44  & 0.999  & 2897.77  & 0.999   & 88.45   & 0.998   & 52.78  & 0.976 & 4085.65 & 36 
   & 0.992  &    458.21 \\
\cline{3-14}
  \hline 
  {Linear}   & 900 & 5.04   & 0.494  & 1197.99 
 &  0.494 &    3.79    & 0.494  &    3.52 & 0.792 & 3546.69 &12  & 0.010  & 0.02 
 \\
 2D  & 1225  & 8.20 & 0.800  & 1389.78  & 0.800   & 7.22 & 0.800 & 6.64  &0.844 & 3579.58 & 20 & 0.010 & 11.08 
\\ 
 \textit{ Non-Convex}  & 1444 & 9.18  &  0.921 & 1545.45 &  0.921  & 11.65  & 0.921  & 10.12  & 0.855 & 3589.13 & 24 & 0.023   & 37.89    
  \\
  & 2926 & 47.98  & 0.927  & 3161.56  & 0.927 &  98.36 & 0.927  &  20.86 &0.928 &3599.85  & 26 & 0.034   & 62.88 
  \\
  & 5890 &  179.94  & 0.929  & 8191.65  &  0.929 & 458.44   &  0.929  & 133.84  & 0.929&3675.77 & 30 &  0.075 & 196.85  \\
     & 11818 & 477.45  & -  & \texttt{TO}  & 0.936  & 1875.49   & 0.936  & 842.67 & 0.931&3744.23    & 34 & - & \texttt{OM}  
     \\
         & 24336 &  987.65  & - & \texttt{TO}  &  0.938  & 4441.55   & 0.938   & 3099.30  & 0.936 & 4234.56   & 36 & - &\texttt{OM}  
        \\
\cline{3-14}
 \hline
{Pendulum}  & 120 & 6.37 & 1.00 & 0.51  &  0.99  & 5.84  & 0.989  & 3.75& 0.999 & 4242.89 & 4  & 0.999 & 7.71   \\
2D  &    240 & 18.33 &  1.00 & 6.08 & 1.00    & 14.88 & 0.990 & 9.99& 0.999 & 4457.82 & 4 & 0.999 & 34.96  \\
 &    480 &  37.84  & 1.00 &29.39   &  1.00  &43.42  & 0.999  & 17.88  & 0.999 & 4675.12 & 4 & 0.999 & 187.60 
 \\ \cline{3-14}
  \hline
  {Unicycle}  &  1250 & 1103.42 &  0.750 &  1000.19  &  0.750 & 26.37  & 0.750 & 5.68 & \cellcolor{gray!5} & \cellcolor{gray!5} &  2  & 0.00  & 3110.21  \\
 4D &   1800 & 1756.25  & 0.975  & 1719.58  & 0.975   & 92.26  & 0.975  & 25.78  & \cellcolor{gray!5} & \cellcolor{gray!5}  & 4 & 0.00  &   5451.19  \\
 &  2400 & 2001.11 & 0.998 &  2548.56  & 0.998  & 145.45  &0.998 & 55.59 & \cellcolor{gray!5} & \cellcolor{gray!5}& 6  & - & \texttt{OM}  \\
  \hline
   Quadrotor  &  7865 & 80.80  &0.770  & 9906.67 & 0.770 & 1174.49 & 0.770   & 2589.56 & \cellcolor{gray!5} & \cellcolor{gray!5} 
    &  2  & 0.584  & 0.20  \\
6D &  15625  & 160.61 & - &\texttt{TO}  & 0.901  & 3788.98& 0.901  & 3258.87& \cellcolor{gray!5} & \cellcolor{gray!5} &  8  & 0.900 & 8628.58 
   \\ 
   \textit{Convex} &    46656  & 458.59 & - & \texttt{TO}  & -& \texttt{TO}& 0.912   & 9542.75  & \cellcolor{gray!5} & \cellcolor{gray!5}  & 12  &-&\texttt{OM} \\
  \hline
   Quadrotor  &  15625&  188.10  & -&\texttt{TO} & 0.670 & 3845.25   & 0.670 & 3478.31 &  \cellcolor{gray!5} & \cellcolor{gray!5} &  4  & 0.00  & 4.91  \\ 
     6D &    31250& 395.59 & -&\texttt{TO} & 0.810 & 9548.78 & 0.810 & 5878.28 & \cellcolor{gray!5} & \cellcolor{gray!5} &  8 & 0.00  & 8715.54  \\ 
  \textit{Non-Convex} &   46656  & 506.99  &  - & \texttt{TO} &-&\texttt{TO} & 0.900  & 9789.54  & \cellcolor{gray!5} & \cellcolor{gray!5} & 12  & -& \texttt{OM} 
  \\  \hline
   Quadrotor  & 65536  & 845.44  &  - & \texttt{TO}& -& \texttt{TO} & 0.500  & 19377.90 &\cellcolor{gray!5} & \cellcolor{gray!5} & 2 & 0.00  & 14830.23  \\ 
    8D & 128000 & 2530.74  & -& \texttt{TO} & - & \texttt{TO} &  0.560 & 39132.59 & \cellcolor{gray!5} & \cellcolor{gray!5}  & 4   & - & \texttt{OM} \\ 
\bottomrule

\end{tabular}
}
\text{$^{*}$ For linear systems that use \gls{sos} optimization, the state space is not partitioned. Further, synthesis methods for \gls{sos} and }
\text{ \hspace{0.5mm} \gls{nbf} do not use bounds on the transition kernel.
}
\end{table*}

To show the power of PWC-SBFs and efficacy of our proposed synthesis methods, we study their performance on a set of seven benchmark problems, consisting of stochastic systems with linear and nonlinear dynamics and varying dimensionality, from 2D to 8D.
We also compare the performance of PWC-SBFs against state-of-the-art continuous SBF techniques, namely \gls{sos} and \gls{nbf}.

Our implementations of all methods, except \gls{nbf}, including the three proposed PWC-SBF synthesis methods, are in \texttt{Julia}. The code is publicly available for download on GitHub\footnote{
\url{https://github.com/aria-systems-group/StochasticBarrier.jl}.}.  All the computations for the benchmarks were performed on a computer with 3.9 GHz 8-Core CPU and 128 GB of memory.

The stochastic systems we considered are:
\begin{itemize}
    \item 2D linear stochastic system 
    with a (1) convex $X_\safe$, and (2) non-convex $X_\safe$,
    \item 2D pendulum \gls{nndm} adapted from \cite{mazouz2022safety},
    \item 4D nonlinear unicycle model under a hybrid feedback control law,
    \item 6D linearized lateral and vertical quadrotor guidance model with a hybrid controller with a (1) convex $X_\safe$, and (2) non-convex $X_\safe$,
    \item 8D linearized lateral and longitudinal quadrotor guidance model with a hybrid controller. 
\end{itemize}

A summary of the results is shown in Table~\ref{table:results}. 
\ifarxiv
More details on the results are presented in Table~\ref{table:results_full} in Appendix~\ref{appendix: detailed table}.
\fi
For all case studies, $N = 10$ time steps. \texttt{TO} and \texttt{OM} denote time-out ($\tau > 45 \times 10^3 s$)  and out-of-memory, respectively.
Due to limitations pertaining to memory or hybrid control, the \gls{nbf} is only run for the 2D linear system and the pendulum NNDM. 
It is further denoted that all \gls{nbf} experiments are run for 150 epochs, with 400 iterations per epoch. Each \gls{nbf} architecture consists of 2 hidden layers with 32 ReLu activated neurons per layer.
For a relevant comparison, it is noted that NBF uses the same initial partitioning as PWC.
Finally, it is noted that the GD approach uses a decay rate of $0.999$, and a termination condition when the difference in $\max \beta$ is less than $1e^{-6}$ among $50$ iterations.

We provide detailed discussions on the experimental setup and obtained results below.
For case studies with non-convex $X_\safe$, we use the notion of \emph{obstacle} to refer to the unsafe sets.  All the obstacles are hyper-rectangles defined 
by a center point
$c = (c_1, \ldots, c_n) \in \mathbb{R}^{n}$ and the half-length along each dimension denoted by ${\epsilon}_i$. Formally,  
$$\C(c,\epsilon) = \big \{ (x_1, ..., x_n) \in \reals^n  : |x_i - c_i| \leq \epsilon_i \;\; \forall 1\leq i \leq n \big\}.$$

\subsection{2D Linear System}
We considered a linear stochastic system in $\reals^2$ with dynamics
$$\mathbf{x}(k+1) = 0.5I \cdot \mathbf{x}(k) + \mathbf{v},$$
where $I$ is the identity matrix, and $\pv \sim \N(0,10^{-2}I)$.
The initial set  $X_0 = [-0.8, -0.6] \times [-0.2, 0.0]$. 
We consider two cases for $X_\safe$, both defined based on the set $X_D = [-1, 0.5] \times [-0.5, 0.5]$
as shown in Figure~\ref{fig:env}: 
\begin{enumerate}
    \item Convex safe set: $X^{conv}_\safe = X_D$, and
    \item Non-convex safe set: $X^{ncon}_\safe = X_D \setminus \bigcup_{i=1}^2 \C_i(c_i,\epsilon_i)$, where
    \begin{align*}
        &c_1 = (-0.55, +0.30), \quad \epsilon_1 = (0.02, 0.02), \\
        &c_2 = (-0.55, -0.15), \quad \epsilon_2 = (0.02, 0.02).
    \end{align*}
\end{enumerate}

\paragraph*{Convex $X^{conv}_\safe$ Case}
Observe in Table~\ref{table:results} that all three PWC-SBF methods outperform the state-of-the-art continuous SBF techniques \gls{sos} and \gls{nbf}. It takes a degree $36$ polynomial SBF to obtain a result similar to that of the \gls{pwc} methods using just $K = 64$, despite the absence of obstacles (convexity of $X_\safe$). 
This is mainly attributed to the fact that initial set is not configured in a symmetric fashion, making curve fitting difficult for even simple 2-dimensional problems. 

The main limitation of \gls{nbf} in this experiment pertains to the training time. 
Note that it takes
the continuous SBF methods 3-4 orders of magnitude larger in computation time to arrive to a similar value of $\low{P}_\safe$ as the PWC methods.
In addition, 
the continuous SBF methods never reached $\low{P}_\safe = 0.999$. 


Among the PWC-SBF methods, the Dual LP and CEGS always obtain the same $\low{P}_\safe$, as they are both exact optimal methods. However,  it takes CEGS at least an order of magnitude less time to arrive to this optimal value.  That is due to the large number of the decision variables in the Dual LP, which increases as $K$ increases.  The GD approach is the fastest method; however, its 
hyperparameters are tricky to design. For this case study, GD terminated before convergence to the true optimal $\low{P}_\safe$.

\paragraph*{Non-convex $X^{ncon}_\safe$ Case}
For this case, the power of the PWC-SBFs are even more evident (see Figure~\ref{fig:sos_pwc}). Due to the fact that the obstacles are within the safe set, the SOS approach has difficulties fitting a function. This forces $B(x) \geq 1$ in the entire state space, as is evident from Figure~\ref{fig:barrier_sos}. This phenomenon does not improve much for a significant increase in the degree of the SOS polynomial, as it is an inherent limitation of the appraoch. 

On the other hand, our methods effectively tackle this problem by assigning $b_i = 1$ to the regions that overlap with the obstacles, and optimize over the remainder of the state space. This results in the \gls{pwc} as depicted in Figure~\ref{fig:barrier_pwc}. The difference is likewise highly notable in terms of the safety probability $\low{P}_\safe$, where SOS can only guarantee safety probability of $0.075$ for the polynomial degree 30, where two of PWC-SBF methods provide up to $0.938$ safety probability.

Note that for large values of $K$, the dual \gls{lp} is unable to terminate, due to the scope of the required convex optimization. A similar fact is observed for high degree SOS barrier polynomials. The \gls{nbf} method performs well in terms of safety probability, at lower number of partitions even better than PWC SBF. Nonetheless, the \gls{pwc} SBFs are superior in terms of computation time.


\subsection{2D Pendulum NNDMs}
\begin{figure*}[t!]
    \centering
    \begin{minipage}{\textwidth}
        \begin{subfigure}[t]{0.33\textwidth}
            \centering
            \includegraphics[width=1\linewidth]{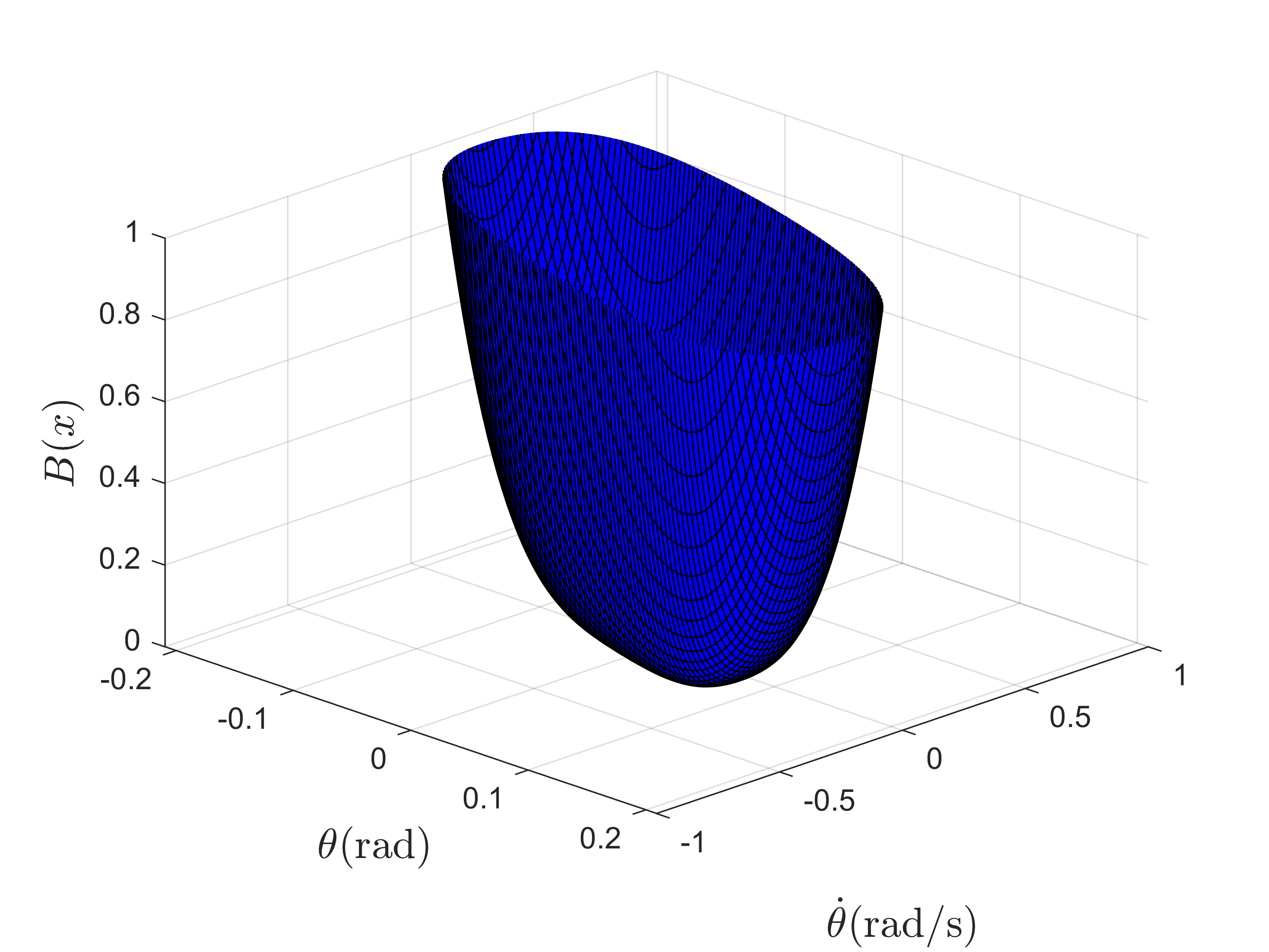}
            \subcaption{Degree-8 SOS-polynomial}
            \label{fig:pendulum_sos}
        \end{subfigure}%
             \begin{subfigure}[t]{0.33\textwidth}
            \centering
            \includegraphics[width=1\linewidth]{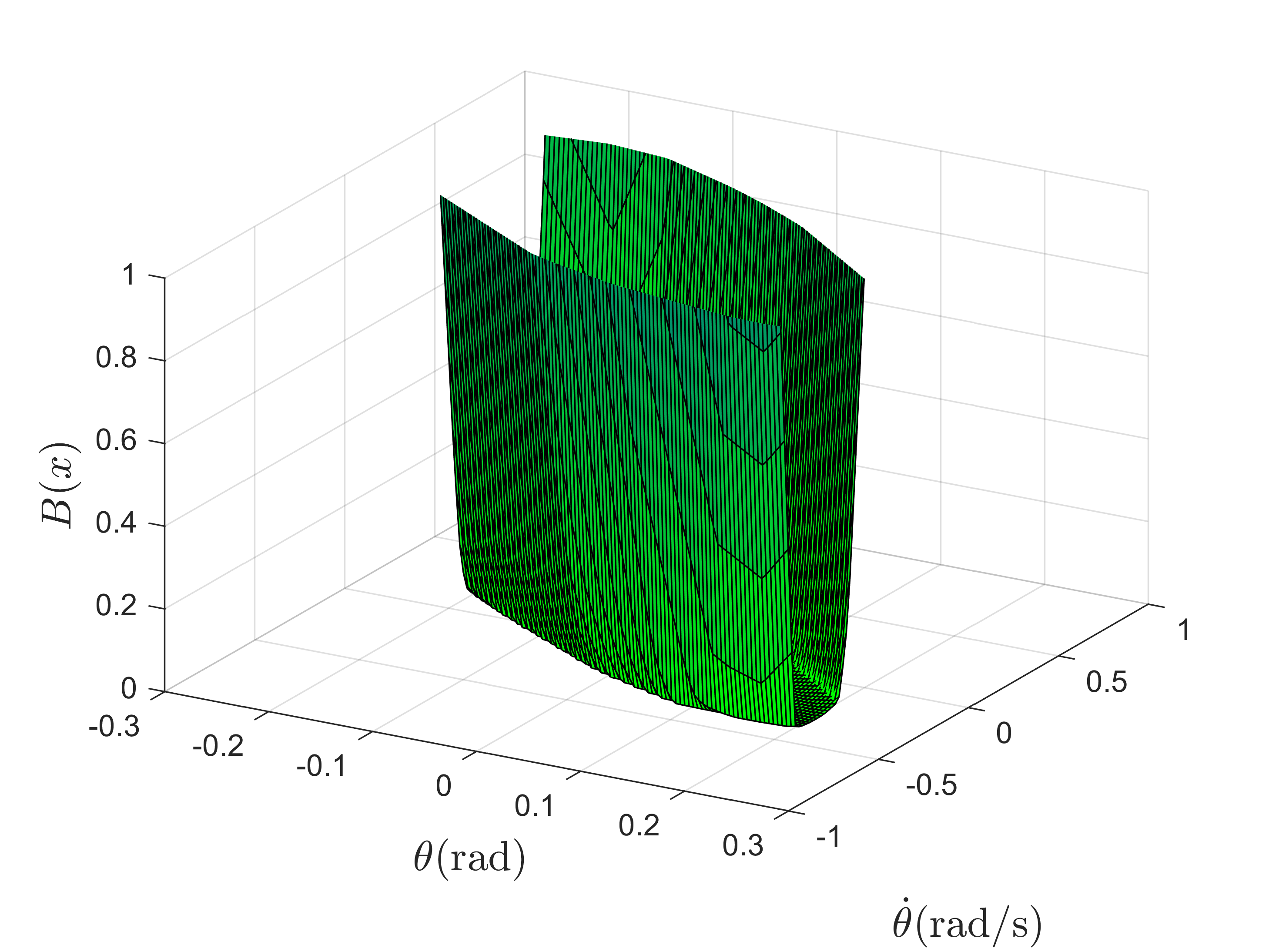}          
            \subcaption{\gls{nbf} with $|K|= 480$.}
            \label{fig:pendulum_nbf}
        \end{subfigure}%
        \begin{subfigure}[t]{0.33\textwidth}
            \centering
            \includegraphics[width=1\linewidth]{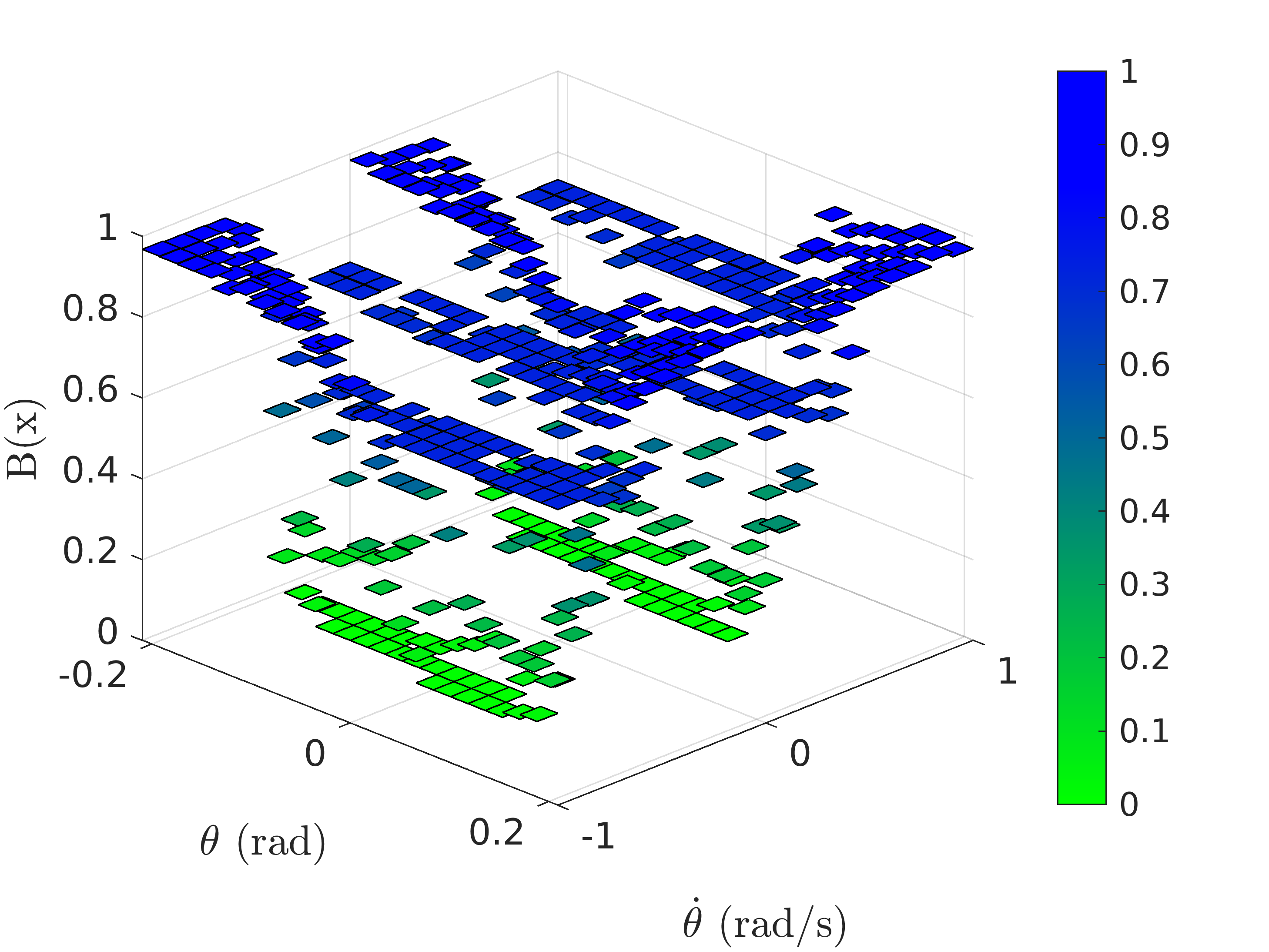}
        \subcaption{\gls{pwc} (dual \gls{lp}) with $|K|= 480$.}
        \label{fig:pendulum_pwc}
        \end{subfigure}%
    \end{minipage}
    \caption{
    SBF plots for the pendulum NNDM case study. 
    }
    \label{fig:pendulum_barrier}
\end{figure*}

In this case study,  we consider the NNDM from \cite{mazouz2022safety} with dynamics $\px = f^{NN}(x) + \pv$, where $f^{NN}$ is a NN trained on a Pendulum agent with a fixed mass $m$ and length $l$ in closed-loop with a NN controller.  The actuator is limited by $u \in [-1, 1]$ and the dynamics of the pendulum are
\begin{align*}
    \dot{\theta}_{k+1} & = \dot{\theta}_{k} + \frac{3g}{2l}\sin(\theta_{k})\delta t^2 + \frac{3}{ml^2}u\delta t^2,\\
    \theta_{k+1} & = \theta_k + \dot{\theta}_{k+1} \delta t
\end{align*}
For the details on the trained model, we refer to \cite{mazouz2022safety}.
For this system, noise $\pv\sim \N(0,10^{-4}I)$. 
The safe set is $\theta \in [-\frac{\pi}{15}, \frac{\pi}{15}]$ and $\dot{\theta} \in [-1, 1]$ and the initial set is $\theta, \dot{\theta} \in [-0.01, 0.01]$. 

We note that since for this particular case study the dynamics are given by a NN, the SOS approach also uses the partitioning of size $K$, where local linear relaxations of $f^{NN}$ are performed. For more information on this SOS formulation of the SBF see \cite{mazouz2022safety}.

As can be observed from Table~\ref{table:results}, the \gls{pwc}, \gls{sos} and \gls{nbf} methods perform similarly in terms of safety probability. It is noted that \gls{sos} performs well due to the fact that the initial set is centered in the state space. 
In Figure~\ref{fig:pendulum_barrier}, we show the plots of the \gls{sos}, \gls{nbf} and \gls{pwc} (dual approach) SBFs. 
In Figure~\ref{fig:pendulum_sos}, we observe that the obtained SOS polynomial is an over-fit, as the values near the boundary of the safe set by-far exceed one. We observe a similar yet less excessive pattern for the \gls{nbf} in Figure~\ref{fig:pendulum_nbf}. For the \gls{pwc} SBF the barrier in the domain is less then or equal to 1. 
Among the PWC SBF methods, a general trend is observed: the dual LP is the slowest algorithm and the GD the fastest.  It is also worth noting that that the GD method terminates with a $\low{P}_\safe$ that is slightly below the optimal value, indicating the difficulty of designing the termination criterion.

\subsection{4D Unicycle}

\begin{figure}[b!]
    \centering
    \includegraphics[width=1\linewidth]{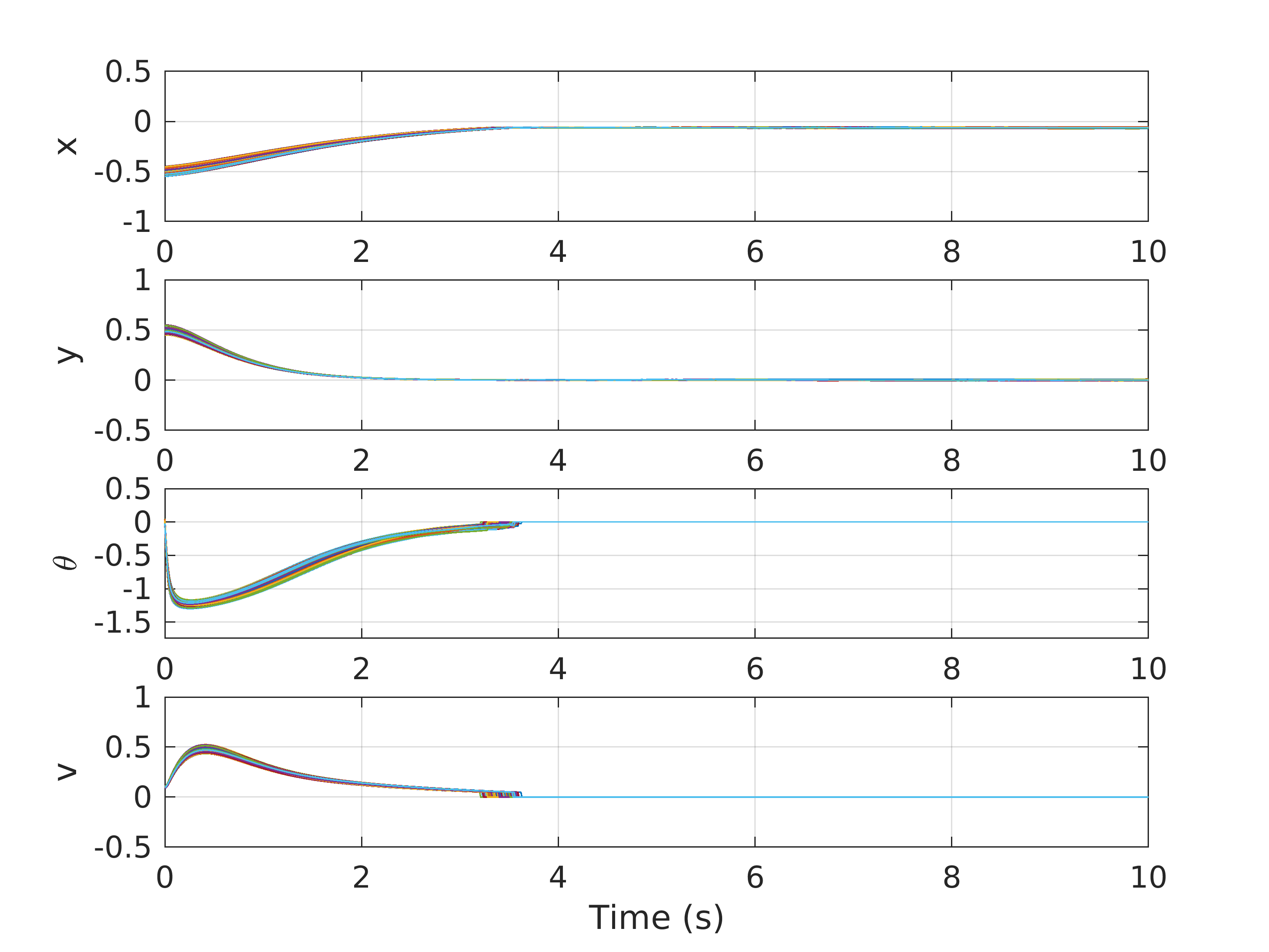}
    \caption{
    Monte Carlo simulation ($10^{5}$ instances) for the unicycle model with $t = 10 s$. 
    }
    \label{fig:unicycle_sim}
\end{figure}

In this case study, we consider a wheeled mobile robot with the dynamics of a unicycle
\begin{align*}
    \dot{x} = v \cos \theta, \quad
    \dot{y} = v \sin \theta, \quad
    \dot{\theta} = \omega, \quad 
    \dot{v} = a,
\end{align*}
where $x \in [-1.0, 0.5]$ and
$y \in [-0.5, 1.0]$ are the Cartesian position, $\theta \in [-1.75, 0.5]$ is the orientation with respect to the $x-$axis, and $v \in [-0.5, 1.0]$ is the speed. The inputs are steering rate $\omega$ and acceleration $a$.
We design a feedback linearization controller according to \cite{de2000stabilization} coupled with an LQR stabilizing controller, making this a hybrid model.
We use $\Delta t = 0.01$ to obtain a discrete time dynamics, using the Euler method.  We add noise $\pv \sim \N(0, 10^{-4}I)$ to capture the time-discretization error that is inherent to the Euler method.

Figure~\ref{fig:unicycle_sim} shows $10^5$ Monte Carlo simulations of the trajectories of the unicycle from the initial set defined by a ball centered at $(-0.5, -0.5, 0, 0)$ with radius $0.01$.
These simulated trajectories suggest that the system is stable under the hybrid control law. This is in-line with the results in Table~\ref{table:results} for the \gls{pwc} methods, especially for $K = 2400$, for which $\low{P}_\safe = 0.998$. Note that the SOS approach highly suffers for this nonlinear case study, with $\low{P}_\safe = 0$, despite very long computation times. This further emphasizes the efficacy of our methods. Observe that the GD approach terminates with the optimal $\low{P}_\safe$ value in this case study. We finally note that since the controller is hybrid, current \gls{nbf} formulations cannot be employed for this problem.

\subsection{6D and 8D Quadcopter Systems}


\begin{figure*}[t!]
    \centering
    \begin{minipage}{\textwidth}
        \begin{subfigure}[t]{0.33\textwidth}
            \centering
            \includegraphics[width=1\linewidth]{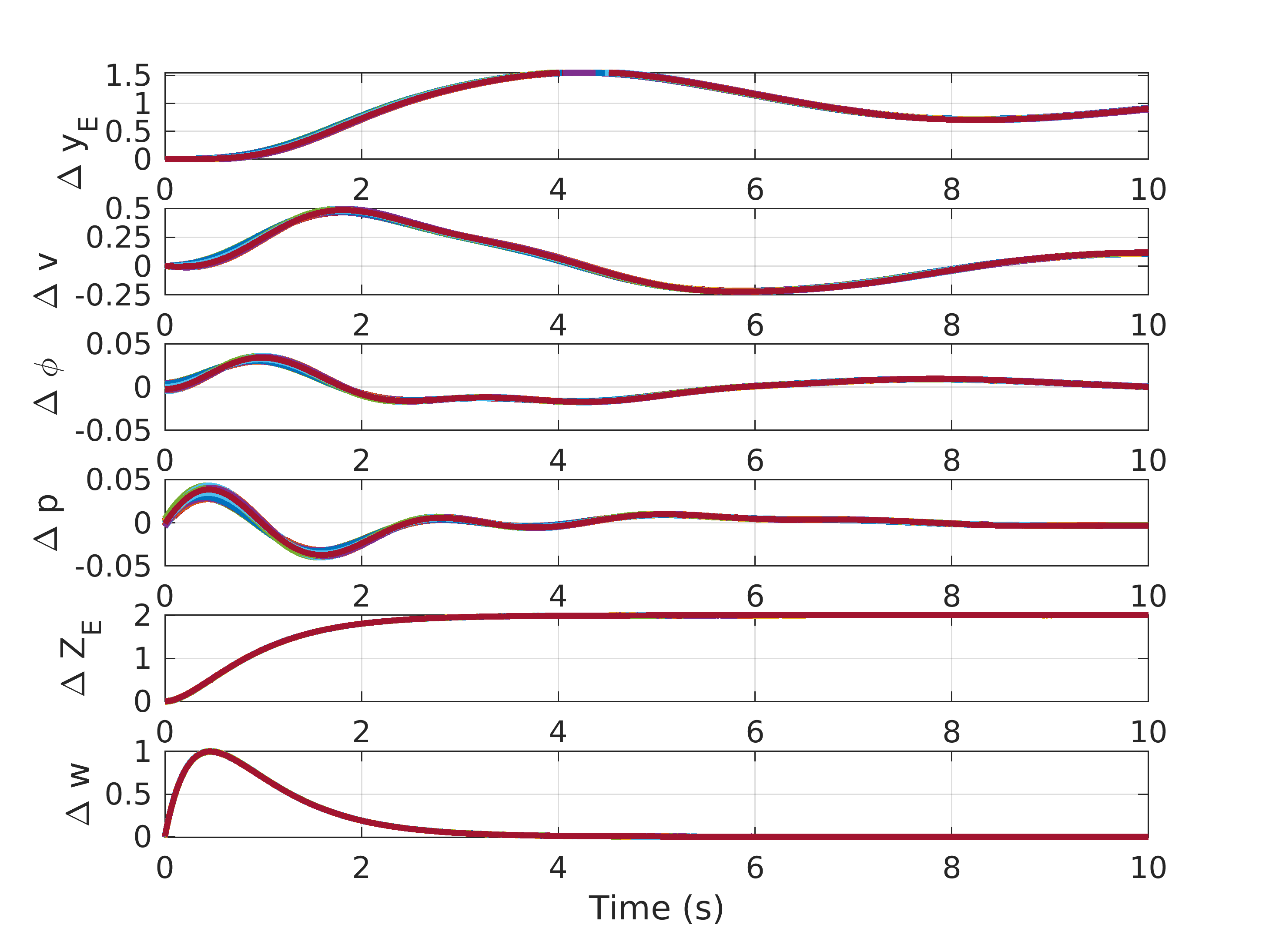}
            \subcaption{State Evolution}
                   \label{fig:quadrotor_2d}
        \end{subfigure}%
             \begin{subfigure}[t]{0.33\textwidth}
            \centering
            \includegraphics[width=1\linewidth]{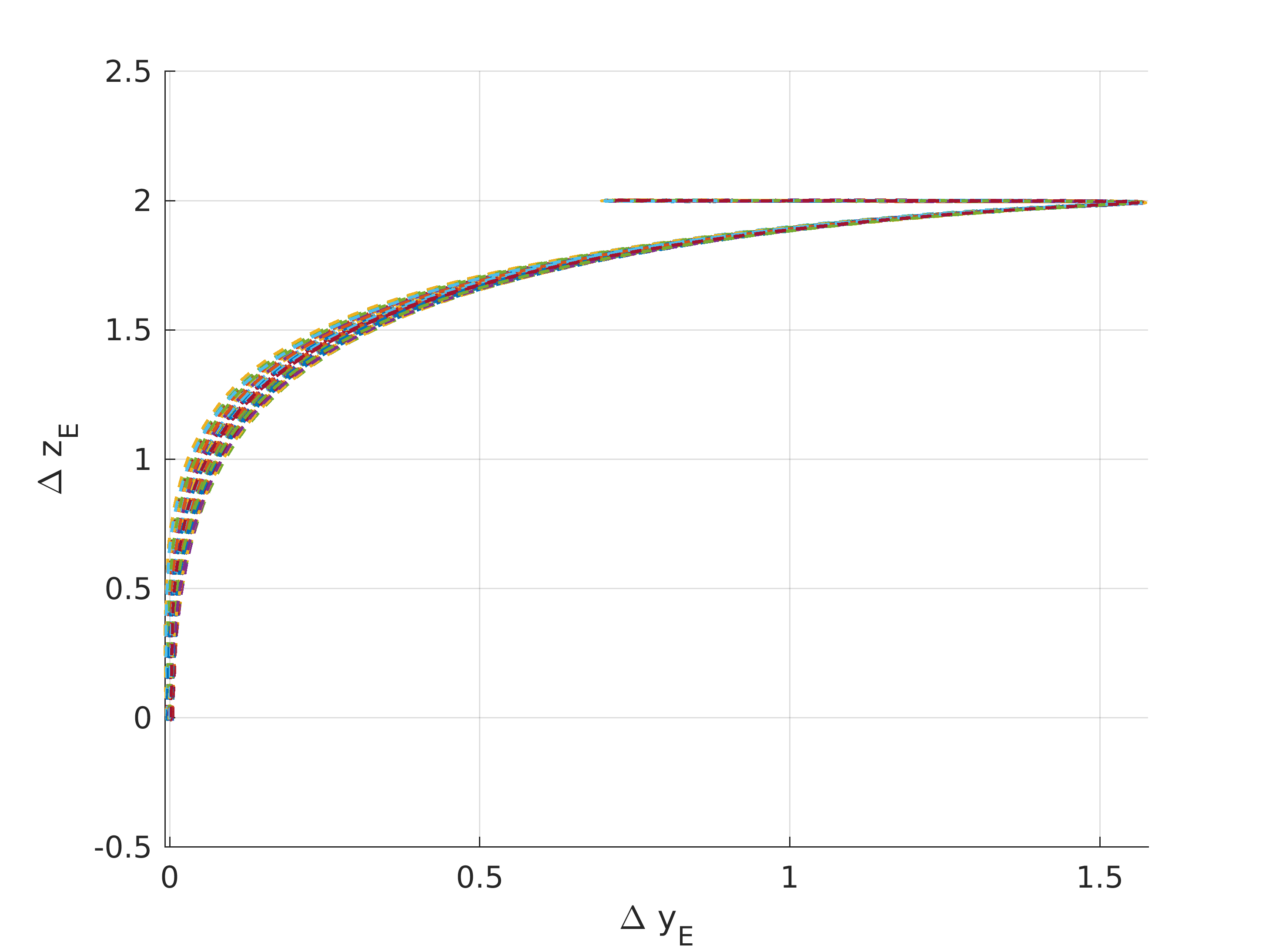}          
            \subcaption{$y$-$z$ Plane}
            \label{fig:quadrotor_planar}
        \end{subfigure}%
        \begin{subfigure}[t]{0.33\textwidth}
            \centering
            \includegraphics[width=1\linewidth]{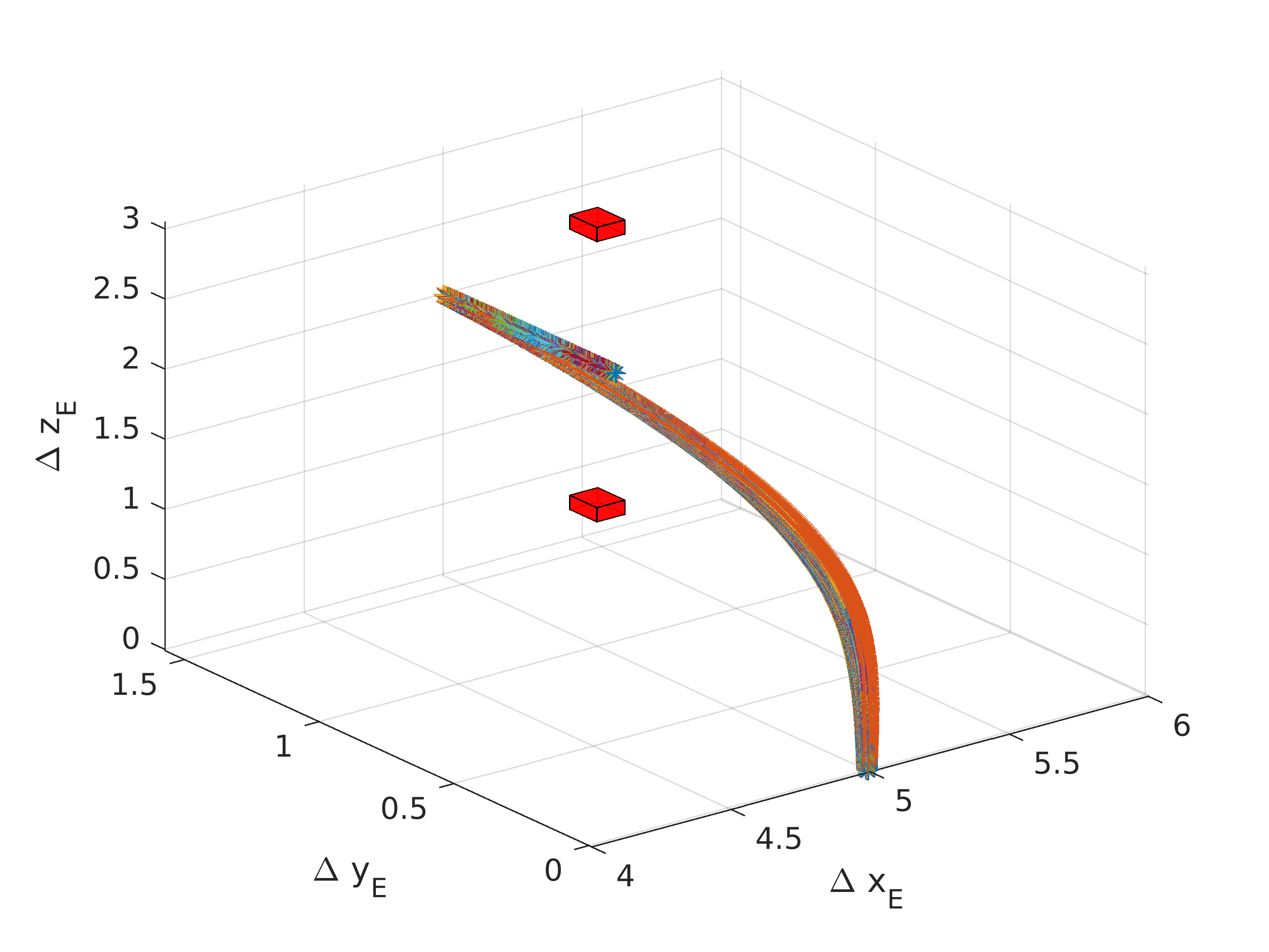}
            \subcaption{3D Workspace}
        \label{fig:quadrotor_3d}
        \end{subfigure}%
    \end{minipage}
    \caption{
    Monte Carlo simulations ($10^5$ instances) for the 6D quadcopter model with lateral-vertical dynamics for $10s$. Trajectories are visualized in (a) state vs. time plot, (b) $y$-$z$ Plane, and (c) 3D workspace.  The red boxes in (c) are the obstacles. The control laws stabilize the quadrotor around the equilibrium point $x_{\text{eq}} = (1, 0, 0, 0, 2, 0)$. 
    }
    \label{fig:quadrotor_trajectory}
\end{figure*}

To evaluate scalability of the proposed PWC-SBF methods, we considered a set of high-dimensional systems.
Specifically, we considered a quadrotor guidance model where the dynamics consist of lateral and vertical
dynamics and the longitudinal and spin variables are constrained, and then lateral and longitudinal motion model by fixing the vertical and spin variables. 

\subsubsection{6D Lateral-Vertical Model}
The lateral closed-loop state space model is defined as
\begin{align*}
    &\Delta \dot{y}_{E} = \Delta v, &&
    \Delta \dot{v} = g \Delta \phi, \\
    & \Delta \dot{\phi} = \Delta p, &&
    \Delta \dot{p} = \frac{1}{I_x} \Delta L_c,
\end{align*}
where 
$g$ is gravity, and $I_x$ is the inertia about the $x-$axis. State $y_{E} \in [-0.5, 2.0]$ is the $y-$position, and $v \in [-1.0, 1.0]$ is the corresponding velocity. The roll angle $\phi \in [-0.1, 0.1]$, the roll rate $p \in [-0.1, 0.1]$, and $L_c$ is the roll control moment.
The control law follows directly from
\begin{multline*}
    \Delta L_c = -k_{1} \Delta p - k_2 \Delta \phi
    - k_3 \Delta v - \\
    k_3 k_4 \Delta y_{E}
    + k_3 k_4 \Delta y_{E, r},
\end{multline*}
where $k_i$ denotes the gain for $i \in \left \{1, \ldots, 4 \right \}$, and $\Delta y_{E, r}$ is the reference position. 

The lateral model is appended with a guidance model for vertical motion given by
\begin{align*}
    \Delta \dot{z}_{E} = \Delta w, \qquad 
    \Delta \dot{w} = \frac{1}{m} \Delta Z_c,
\end{align*}
where $m$ denotes the quadrotor mass, $z_{E} \in [-0.5, 3.0]$ denotes the $z$-position and $w \in [-0.5,1.5]$ is the yaw rate. The control law $\Delta Z_c$ is defined as 
\begin{align*}
    \Delta Z_c = -k_{1} \Delta z_{E} - k_2 \Delta w - Fr
\end{align*}
where $Fr$ denotes the reference signal for the vertical dynamics.
The full dynamics constitutes a 6D linear model. 
Using $\Delta t = 0.01$, we obtained a discrete-time dynamics via the Euler method, to which  we added noise $\pv \sim \N(0, 10^{-4}I)$. 

The designed control laws stabilize the quadrotor around the equilibrium point $x_{\text{eq}} = (1, 0, 0, 0, 2, 0)$. 
Monte Carlo simulations of $10^5$ trajectories of this model is shown in Figure~\ref{fig:quadrotor_trajectory}.
The 3D workspace (environment) is shown in Figure~\ref{fig:quadrotor_3d}.

We considered two safe sets:
\begin{enumerate}
    \item Convex safe set: $X^{conv}_\safe$ is defined by the given ranges of of the variables above,
    \item Non-convex safe set: $X^{ncon}_\safe = X^{conv}_\safe \setminus \bigcup_{i=1}^2 \C(c_i,\epsilon_i)$, where
        \begin{align*}
            &c_1 = (1, 0, 0, 0, 1, 0), \qquad && \epsilon_1 = 0.01 \cdot \indicator{6 \times 1}\\
            &c_2 = (1, 0, 0, 0, 2.75, 0),  && \epsilon_2 = 0.01 \cdot \indicator{6 \times 1}. 
        \end{align*}    
    Figure~\ref{fig:quadrotor_3d} depicts the obstacles.
\end{enumerate}

\paragraph*{Convex $X^{conv}_\safe$ case}
Table~\ref{table:results} shows that the two \gls{pwc} methods (namely, \gls{cegs} and \gls{gd}) outperform \gls{sos}, both in terms of computation time, as well as safety probability certification.  It is important to note that systems with linear dynamics and convex $X_\safe$ are an ideal setup for \gls{sos} optimization. However, it has a lot of difficulty with the dimensionality of the system as the results suggest. For this case study, an SOS polynomial of degree greater than 8 runs out of memory.  
Finally, we were not able to run \gls{nbf} for this problem since it cannot handle hybrid controllers and suffers from scalability.
\paragraph*{Non-convex $X^{ncon}_\safe$ case}
In this case, it is even more evident that PWC-SBFs can be more powerful than continuous SBFs.
The \gls{sos} method provides $\low{P}_\safe = 0.00$, whereas the \gls{gd} method gets to $\low{P}_\safe = 0.900$. Observe that the dual \gls{lp} method times out, which is also true for the \gls{cegs} method when $K > 4.5\times 10^4$. The \gls{gd} approach is able to handle this, showing the capability of \gls{pwc} SBFs in scalability.

\subsection{8D Lateral-Longitudinal Dynamics}
Here, we aim to test the performance boundary of PWC-SBFs in terms of scalability.  We consider an 8D model by combining the lateral dynamics of the quadcopter with its longitudinal and fixing vertical and spin variables.
In a similar fashion as lateral, the longitudinal guidance is constructed under the dynamics
\begin{align*}
    & \Delta \dot{x}_{E} = \Delta u, &&
    \Delta \dot{u} = -g \Delta \delta, \\ 
    & \Delta \dot{\theta} = \Delta q, &&
    \Delta \dot{q}= \frac{1}{I_y} \Delta M_c, 
\end{align*}
where 
$g$ is gravity, and $I_y$ is the inertia about the $y-$axis. State $x_{E} \in [-0.5, 4.0]$ is the $x-$position, and $u \in [-0.5, 1.5]$ is the corresponding velocity. The pitch angle $\theta \in [-0.1, 0.1]$, the pitch rate $q \in [-0.1, 0.1]$, and $M_c$ is the pitch control moment.
The control law follows directly from
\begin{multline*}
    \Delta M_c = -l_{1} \Delta q - l_2 \Delta \theta
    - l_3 \Delta u - \\
    l_3 l_4 \Delta x_{E}
    + l_3 l_4 \Delta x_{E, r},
\end{multline*}
where $l_i$ denotes the gain for $i \in \left \{1, \ldots, 4 \right \}$, and $\Delta x_{E, r}$ is the reference position. 



The results are shown in the bottom 2 rows of Table~\ref{table:results}.
For this system, the SOS approach is unable to provide any safety guarantees, i.e., $\low{P}_\safe = 0$. The dimensionality naturally has a major effect on the performance of the SOS optimizer. 
The dual LP and CEGS likewise suffer, pertaining to the dimensionality and the partitioning of the state space. Nonetheless, the GD approach is still able to handle the large scale optimization, and can provide a lower safety probability of $\low{P}_\safe =  0.560$. While this is not near the statistical safety threshold obtained through simulation ($\tilde{P}_{\safe,{sim}} \approx 0.95 $), it 
clearly shows that the proposed PWC-SBF method outperforms the current
state-of-the-art SBF methods.

\section{Conclusion}
In this work, we introduce a formulation for piecewise (PW) stochastic barrier functions (SBFs).  We discuss the flexibility they provide at the cost of increased computational cost.  To lay the groundwork, we focus on their simplest form, namely, PW constant (PWC) functions.  We show that the synthesis of optimal PWC SBFs reduces to a constrained minimax optimization problem that includes bilinear terms.  We further show that PWC SBFs can obtain the same or even higher probabilistic safety guarantees than continuous SBFs.

We also propose three efficient computational methods for solving the minimax problem by relaxing bilinear terms.  This is achieved by (i) dual formulation in form of a linear program (LP), (ii) splitting the problem into two separate iterative LPs, where one synthesizes candidate function and the other provide counter examples, and finally (iii) designing a convex loss function that precisely captures the  objective function of the minimax problem, enabling the use of gradient descent method.
To show the expressivity of PWC SBFs and scalability of the proposed computational methods, we provide
extensive evaluations on a range of benchmark problems.  We also compare performance of the proposed methods against state-of-the-art \gls{sos} optimization and \gls{nbf} learning.  The results clearly show the efficacy and superiority of PWC SBFs. The PWC SBFs provide a new approach to ensuring safety for stochastic systems.

This study raises several interesting research questions for future investigations.  One is how to perform adaptive refinement of the pieces of the PWC SBF to reduce computational overhead.  Another interesting question is on the complexity of PW linear (and nonlinear) SBFs and whether they provide more expressivity (better safety guarantees) than PWC SBFs.


\bibliographystyle{ieeetr}        
\bibliography{bibliography}           

\appendix

\ifarxiv
\section{Extensive Results}
\label{appendix: detailed table}
The extensive results pertaining to the case studies are presented in Table \eqref{table:results_full}.

\begin{sidewaystable*}[t]
\caption{ This table provides extensive results pertaining to the case studies for safety verification based on \emph{piecewise constant} and \emph{continuous} stochastic barriers, as briefly established 
in Table~\ref{table:results}.
}
\label{table:results_full}
\resizebox{1\textwidth}{!}{%
\begin{tabular}{@{}c| c c| ccc | cccc | cccc || cccc| ccccc 
!{\color{white}\vline}
@{}}
\toprule
 &      \multicolumn{15}{c }{\underline{Piecewise Constant Stochastic Barriers}} 
 &  \multicolumn{5}{c }{\underline{Continuous Stochastic Barriers}} 
 \\
 &     & &  \multicolumn{3}{c}{ Dual Linear Program}  & \multicolumn{4}{c }{Counter-Example Guided} & \multicolumn{4}{c||}{Gradient Descent} & \multicolumn{4}{c|}{Neural Barrier Function} & \multicolumn{4}{c}{Sum-of-Squares  } \\ \cline{4-23}
  {Model} & {$|K|$} & $\mathcal{T}_{p} (s) $  &$\beta$ & $P_\safe$ &$\mathcal{T}_{0}(s)$ & 
  $\beta$ & $P_\safe$  & $\mathcal{I}$  & $\mathcal{T}_{o} (s) $& 
  $\beta$ & $P_\safe$  & $\mathcal{I}$  & $\mathcal{T}_{o} (s) $& $\eta$ & $\beta$  &  $P_\safe$ & $\mathcal{T}_{o}$ & Deg & $\eta$ & $\beta$ & $P_\safe$ & $\mathcal{T}_{o} (s) $  \\
  \hline 
  {Linear}  &  64 & 0.02 &  $8.4e^{-4}$  & 0.992 &  0.52
 & $8.4e^{-4}$&  0.992 & 2  &  0.04 &  $4.7e^{-3}$&  0.952 &  50 &   0.04  & $7.1e^{-5}$ & $4.1e^{-2}$ & 0.585 & 3850.93  &4  & $2.4e^{-1}$ & $1.8e^{-2}$ &  0.582 & 0.014 \\
2D  & 225 & 0.31 & $1.4e^{-4}$  & 0.998   & 164.60  & $1.4e^{-4}$ &  0.998 &  2 & 0.44 & $2.6e^{-3}$ & 0.973 & 50 & 0.20 &$2.0e^{-3}$ &$5.8e^{-3}$  & 0.940 & 3991.47 &  8& $2.4e^{-1}$& $1.7e^{-2}$& 0.582 & 0.265 
 \\ 
 \textit{Convex }  & 900  &  8.85 & $1.1e^{-4}$ & 0.999 & 1087.78  & $1.1e^{-4}$  & 0.999 & 2   & 17.93  & $1.0e^{-3}$ & 0.990  & 150 & 7.22 & $1.9e^{-3}$ & $3.7e^{-3}$ & 0.961 & 4025.67 & 30 & $1.4e^{-2}$ &$7.6e^{-4}$  & 0.978  & 151.16   \\
   & 2500  &   41.44 & $8.1e^{-5}$  & 0.999  & 2897.77  & $8.1e^{-5}$  & 0.999 & 2   & 88.45  &  $9.1e^{-4}$ & 0.998  & 175 & 52.78 & $1.7e^{-3}$ &$2.2e^{-3}$ & 0.976 & 4085.65& 36 & $7.0e^{-3}$
  & $5.9e^{-5}$   & 0.992  &    458.21 \\
\cline{3-16}
  \hline 
  {Linear}  & 900 & 5.04  & $5.1e^{-2}$  & 0.494  & 1197.99 
 & $5.1e^{-2}$  &  0.494 &  1 &  3.79  & $5.1e^{-2}$  & 0.494  & 50  &    3.52   & 1.5$e^{-1}$ & 6.1$e^{-3}$& 0.792 & 3546.69 &12  & $9.9e^{-1}$  & $1.0e^{-6}$ & 0.010  & 0.02 

 \\
2D & 1225  & 8.20 & $2.0e^{-2}$ & 0.800  & 1389.78 & $2.0e^{-2}$ & 0.800  & 1 & 7.22 & $2.0e^{-2}$ & 0.800 &  50 & 6.64 & 1.5$e^{-1}$ & 5.4$e^{-3}$& 0.844&3579.58& 20  & $9.9e^{-1}$ &  $2.3e^{-6}$ & 0.010 & 11.08 

\\ 
\textit{ Non-Convex} & 1444 & 9.18 & $7.9e^{-3}$ &  0.921 & 1545.45 & $7.9e^{-3}$& 0.921 &  1 & 11.65 &  $7.9e^{-3}$ & 0.921  & 55 & 10.12  &  1.4$e^{-1}$& 5.2$e^{-3}$& 0.855&3589.13 & 24 & $9.2e^{-1}$ & $5.3e^{-3}$  & 0.023   & 37.89    \\
   & 2926 & 47.98  & $7.2e^{-3}$ & 0.927  & 3161.56 &  $7.2e^{-3}$  & 0.927 & 2 &  98.36 &  $7.2e^{-3}$ & 0.927   & 115&  20.86 &  1.9$e^{-2}$ & 5.3$e^{-1}$ & 0.928 & 3599.85   & 26 & $9.1e^{-1}$ & $5.6e^{-3}$  & 0.034   & 62.88 \\
      & 5890 &  179.94 & $7.1e^{-3}$ & 0.929  & 8191.65 &  $7.1e^{-3}$  &  0.929& 3 & 458.44   & $7.1e^{-3}$   &  0.929  & 285 & 133.84 &1.8$e^{-2}$  &5.3$e^{-1}$ & 0.929& 3675.77 & 30 & $8.7e^{-1}$   & $4.7e^{-3}$ &  0.075 & 196.85 \\
    & 11818 & 477.45  & - & -  & \texttt{TO}&  $6.4e^{-3}$  & 0.936 & 3 & 1875.49  &  $6.4e^{-3}$ & 0.936  & 445 & 842.67 &1.7$e^{-2}$ &5.2$e^{-1}$ & 0.931 & 3744.23 & 34 & - & -  &  - & \texttt{OM} \\
        & 24336 &  987.65 & - & -  & \texttt{TO}& $6.2e^{-3}$  &  0.938  & 2 & 4441.55  &  $6.2e^{-3}$  & 0.938  & 595 & 3099.30 &1.6$e^{-2}$ & 4.8$e^{-3}$& 0.936& 4234.56  & 36 & - & -  &  - & \texttt{OM}  \\
\cline{3-16}
 \hline
{Pendulum}  & 120 & 6.37 & 1.0$e^{-6}$& 1.00 & 0.51  & 1.1$e^{-4}$&  0.99  & 50& 5.84 &  1.1$e^{-3}$ & 0.989 & 15000 & 3.75 & 7.7$e^{-4}$ & 2.1$e^{-5}$ &0.999 & 4242.89 & 4 & 5.9$e^{-5}$ & 2.1$e^{-4}$ & 0.999 & 7.71   \\
 2D &   240 & 18.33 &  1.0$e^{-6}$ &  1.00 & 6.08 & 2.2$e^{-5}$& 1.00  &200   & 14.88 & 1.0$e^{-3}$& 0.990 & 20000 & 9.99 &7.6$e^{-4}$  & 2.1$e^{-5}$ & 0.999 & 4457.82 & 4 & 4.9$e^{-5}$ & 2.1$e^{-4}$ & 0.999 & 34.96  \\
 &  480 &  37.84 & 1.0$e^{-6}$ & 1.00 &29.39  & 2.0$e^{-6}$  &  1.00 &  300 &43.42 & 8.1$e^{-5}$   & 0.999  & 40000  & 17.88  & 5.8e$^{-4}$& 1.9$e^{-5}$ & 0.999 & 4675.12 & 4 & 4.8$e^{-5}$ & 2.0$e^{-4}$ & 0.999 & 187.60\\ \cline{3-16} 
  \hline
  {Unicycle}  &  1250 & 1103.42 & $2.5e^{-2}$ &  0.750 &  1000.19 & $2.5e^{-2}$  &  0.750& 200 & 26.37  & $2.5e^{-2}$ & 0.750 & 300 & 5.68  &\cellcolor{gray!5} & \cellcolor{gray!5}& \cellcolor{gray!5}& \cellcolor{gray!5}& 2  & $9.9e^{-1}$ & 1.0$e^{-6}$ & 0.00  & 3110.21  \\
 4D &  1800 & 1756.25 & $2.5e^{-2}$ & 0.975  & 1719.58  &$2.5e^{-2}$& 0.975  & 200  & 92.26 & $2.5e^{-2}$   & 0.975 & 20000 & 25.78  & \cellcolor{gray!5} & \cellcolor{gray!5}& \cellcolor{gray!5}& \cellcolor{gray!5} & 4 & $9.9e^{-1}$ & 1.0$e^{-6}$ & 0.00  &   5451.19  \\
 & 2400 & 2001.11
 & $2.1e^{-3}$  & 0.998 &  2548.56 &  $2.1e^{-3}$  & 0.998 & 200  & 145.45&  $2.1e^{-3}$  &0.998 &40000  & 55.59 &\cellcolor{gray!5} & \cellcolor{gray!5}& \cellcolor{gray!5}& \cellcolor{gray!5} & 6  & -& -& - & \texttt{OM} \\
  \hline
   Quadrotor &  7865 & 80.80  & $2.3e^{-2}$ &0.770  & 9906.67 & $2.3e^{-2}$ & 0.770 & 450& 1174.49&  $2.3e^{-2}$ & 0.770   & 40000  & 2589.56 &\cellcolor{gray!5} & \cellcolor{gray!5}& \cellcolor{gray!5}& \cellcolor{gray!5}& 2 & $8.8e^{-2}$ & $3.3e^{-2}$  & 0.584  & 0.20 
   \\ 
6D  &  15625  & 160.61 & -& - & \texttt{TO}&  $9.9e^{-3}$ & 0.901 & 575 & 3788.98&  $9.9e^{-3}$ & 0.901  & 50000  & 3258.87   &\cellcolor{gray!5} & \cellcolor{gray!5}& \cellcolor{gray!5}& \cellcolor{gray!5}&   8   & $1.1e^{-3}$  &$9.9e^{-3}$  & 0.900 & 8628.58 \\ 
  \textit{Convex}   &  46656  & 458.59 &- &  - & \texttt{TO}&-  & - & -& \texttt{TO}&  $8.8e^{-3}$ & 0.912  & 120000  & 9542.75  &\cellcolor{gray!5} & \cellcolor{gray!5}& \cellcolor{gray!5}& \cellcolor{gray!5} & 12   & -  &-  &- &\texttt{OM} \\  
  \hline
 Quadrotor &  15625&  188.10 &- & - &\texttt{TO}&$3.3e^{-2}$  & 0.670&595 & 3845.25  & $3.3e^{-2}$  & 0.670  & 55550 & 3478.31 &\cellcolor{gray!5} & \cellcolor{gray!5}& \cellcolor{gray!5}& \cellcolor{gray!5}&  4 & $9.9e^{-1}$& $1.0e^{-6}$ & 0.00  & 4.91 \\ 
  6D &   31250& 395.59 &- & - &\texttt{TO}& $1.9e^{-2}$ & 0.810 & 615 & 9548.78 & $1.9e^{-2}$   & 0.810 & 89850 & 5878.28  &\cellcolor{gray!5} & \cellcolor{gray!5}& \cellcolor{gray!5}& \cellcolor{gray!5}  &  8 & $9.9e^{-1}$& $1.0e^{-6}$ & 0.00  & 8715.54 \\ 
 \textit{Non-Convex}  &  46656  & 506.99  &- &  - & \texttt{TO}&  -&- &- &\texttt{TO}&  $1.0e^{-2}$  & 0.900 & 121500  & 9789.54  &\cellcolor{gray!5} & \cellcolor{gray!5}& \cellcolor{gray!5}& \cellcolor{gray!5} & 12  &  - & - & -& \texttt{OM}\\  
 \hline
   Quadrotor & 65536  & 845.44   & - &  -& \texttt{TO}& - & - & -& \texttt{TO}&$5.0e^{-2}$  & 0.500  & 198550  & 19377.90 &\cellcolor{gray!5} & \cellcolor{gray!5}& \cellcolor{gray!5}& \cellcolor{gray!5} & 2& $9.9e^{-1}$& $1.0e^{-6}$ & 0.00  & 14830.23 \\ 
    8D & 128000 & 2530.74  &- & -& \texttt{TO}& - & - & - & \texttt{TO} & $4.4e^{-2}$  &  0.560 & 256700 & 39132.59 &\cellcolor{gray!5} & \cellcolor{gray!5}& \cellcolor{gray!5}& \cellcolor{gray!5} &   4  & - & - & - & \texttt{OM} \\ 
 \hline
\bottomrule
\end{tabular}
}
\end{sidewaystable*}

\fi

\end{document}